\documentclass[12pt]{article}
\usepackage[T1]{fontenc}
\usepackage[margin=1.3in]{geometry}

\RequirePackage{amsmath,amsthm,amsfonts,amssymb}
\RequirePackage[authoryear]{natbib}
\RequirePackage[colorlinks,citecolor=blue,urlcolor=blue]{hyperref}
\RequirePackage{graphicx}
\usepackage{txfonts}
\usepackage{mathrsfs}  
\usepackage{comment}
\usepackage{caption}
\usepackage{subcaption}
\usepackage{authblk}

\usepackage{url}            % simple URL typesetting
\usepackage{booktabs}       % professional-quality tables
\usepackage{nicefrac}       % compact symbols for 1/2, etc.
\usepackage{microtype}      % microtypography
\usepackage{centernot}

\usepackage{mathtools}
\usepackage{graphicx}
\usepackage{enumerate,xspace}
\usepackage{fancyhdr}
\usepackage{comment}
\usepackage{parskip}
\usepackage{float}
\usepackage{longtable}
\usepackage{tablefootnote}
\usepackage{url}
\usepackage{standalone}
\restylefloat{table}
\usepackage[sans]{dsfont}

\newtheorem{remark}{Remark}

\usepackage{xspace}

\mathtoolsset{showonlyrefs,showmanualtags}

\def\bt{\mathrm{t}}
\def\bx{\mathbf{x}}

\newcommand{\Indicator}{\mathds{1}}

\newtheorem{assumption}{Assumption}

\newtheorem{theorem}{Theorem}[section]
\newtheorem{lemma}[theorem]{Lemma}
\theoremstyle{remark}
\newtheorem{definition}[theorem]{Definition}

\usepackage{stmaryrd}

\title{On the Pointwise Behavior of Recursive Partitioning and Its Implications for Heterogeneous Causal Effect Estimation}
\author[a]{Matias D. Cattaneo\thanks{cattaneo@princeton.edu}}
\author[a]{Jason M. Klusowski\thanks{jason.klusowski@princeton.edu}}
\author[b]{Peter M. Tian\thanks{ptian@twosigma.com}}
\affil[a]{\small Department of Operations Research and Financial Engineering \\ Princeton University}
\affil[b]{Two Sigma Investments, LP}

\date{\normalsize\today}

\begin{document}

\maketitle

\begin{abstract}
    Decision tree learning is increasingly being used for pointwise inference. Important applications include causal heterogenous treatment effects and dynamic policy decisions, as well as conditional quantile regression and design of experiments, where tree estimation and inference is conducted at specific values of the covariates. In this paper, we call into question the use of decision trees (trained by adaptive recursive partitioning) for such purposes by demonstrating that they can fail to achieve polynomial rates of convergence in uniform norm with non-vanishing probability, even with pruning. Instead, the convergence may be arbitrarily slow or, in some important special cases, such as \emph{honest} regression trees, fail completely. We show that random forests can remedy the situation, turning poor performing trees into nearly optimal procedures, at the cost of losing interpretability and introducing two additional tuning parameters. The two hallmarks of random forests, subsampling and the random feature selection mechanism, are seen to each distinctively contribute to achieving nearly optimal performance for the model class considered.
\end{abstract}


\section{Introduction}

As data-driven technologies continue to be adopted and deployed in high-stakes decision-making environments, the need for fast, interpretable algorithms has never been more important. As one such candidate, it has become increasingly common to use decision trees, constructed by adaptive recursive partitioning, for inferential tasks on a predictive or causal model. These applications are spurred by the appealing connection between decision trees and rule-based decision-making, particularly in clinical, legal, or business contexts, as the determination of the output mimics the way a human user may think and reason \citep{berk2020statistical}. Decision trees are ubiquitous in empirical work not only because they offer an interpretable decision-making methodology \citep{murdoch2019definitions, rudin2019stop}, but also because their construction relies on data-adaptive implementations that take into account the specific features of the underlying data generating process. See \cite{Hastie-Tibshirani-Friedman2009_book} for a textbook introduction.

While data-adaptive, rule-based tree learning is powerful, it is not without its pitfalls. In this paper, we provide theoretical evidence of these shortcomings in commonly encountered data situations. Focusing on the simplest possible data generating process (i.e., a homoskedastic constant regression/treatment effect model), we show that decision trees cannot converge faster any polynomial function of the sample size $n$, uniformly over the entire support of the covariates, with non-vanishing probability. Furthermore, when adding honesty to the tree construction, which is often regarded as an improvement over canonical tree fitting \citep{athey2016recursive}, we show that the resulting decision trees can be inconsistent, uniformly over the covariate support, as soon as the depth of the tree is at least a constant multiple of $\log\log(n)$ (e.g., $\log\log(n)\approx 3$ for $n=1 \text{ billion}$ observations).

Our results paint a rather bleak picture of decision trees, if the goal is to use them for statistical learning \emph{pointwise} (or \textit{uniformly}) over the entire support of the covariates; they can produce unreliable estimates even in large samples for the simplest possible statistical model underlying the data generation. Thankfully, in such settings, we are able to show that random forests are provably superior and exhibit optimal performance when the constituent trees do not. This improvement comes at the cost of losing interpretability and introducing two additional tuning parameters (subsample size and number of candidate variables to consider at each node).

To formalize our results, we consider the canonical regression model where the observed data $\{(y_i,\bx_i^T) : i = 1, 2, \dots n\}$ is a random sample satisfying 
\begin{equation} \label{eq:model}
y_i = \mu(\bx_i)+ \varepsilon_i, \qquad \mathbb{E}[\varepsilon_i \mid \bx_i]=0, \qquad \mathbb{E}\big[\varepsilon_i^2 \mid \bx_i\big]=\sigma^2(\bx_i),
\end{equation}
with $\bx_i = (x_{i1}, x_{i2}, \dots, x_{ip})^T$ a vector of $p$ covariates taking values on some support set $\mathcal{X}$. The parameter of interest is the conditional mean response function $\mu(\bx_i) = \mathbb{E}[y_i \mid \bx_i]$, which may be assumed to belong to some smooth, or otherwise appropriately restricted, set of functions. The goal is to use the observed data together with an algorithmic procedure to learn $\mu(\bx)$ for all values of $\bx\in\mathcal{X}$. While there are many ways to grow a decision tree (i.e., a partition of $\mathcal{X}$), our focus throughout this paper will be on the CART algorithm \citep{breiman1984}, by far the most popular in practice.

A decision tree is a hierarchically organized data structure constructed in a top down, greedy manner through recursive binary splitting. According to conventional CART methodology, a parent node $ \bt $ (i.e., a region in $\mathcal{X}$) in the tree is divided into two child nodes, $ \bt_L $ and $ \bt_R $, by minimizing the sum-of-squares error (SSE)
\begin{equation} \label{eq:se}
\sum_{\bx_i \in \bt}(y_{i} - \beta_1 \Indicator(x_{ij} \leq \tau) - \beta_2 \Indicator(x_{ij} > \tau))^2,
\end{equation}
with respect to the child node outputs, split point, and split direction, $(\beta_1, \beta_2, \tau, j)$, with $\Indicator(\cdot)$ denoting the indicator function.

Because the splits occur along values of a single covariate, the induced partition of the input space $ \mathcal{X} $ is a collection of hyper-rectangles.
The solution of \eqref{eq:se} yields estimates $(\hat\beta_1, \hat\beta_2, \hat \tau, \hat\jmath)$, and the resulting refinement of $ \bt $ produces child nodes $ \bt_L = \{\bx \in \bt : x_{\hat\jmath} \leq \hat \tau\} $ and $ \bt_R = \{\bx \in \bt : x_{\hat\jmath} > \hat \tau\} $. The normal equations imply that $\hat\beta_1 = \overline y_{\bt_L} = \frac{1}{\#\{\bx_i \in \bt_L\}}\sum_{\bx_i \in \bt_L} y_i $ and $\hat\beta_2 = \overline y_{\bt_R} = \frac{1}{\#\{\bx_i \in \bt_R\}}\sum_{\bx_i \in \bt_R} y_i  $, the respective sample means after splitting the parent node at $ x_{\hat\jmath} = \hat \tau $, where $\#A$ denotes the cardinality of the set $A$. These child nodes become new parent nodes at the next level of the tree and can be further refined in the same manner, and so on and so forth, until a desired depth is reached. To obtain a maximal decision tree $T_K$ of depth $K$, the procedure is iterated $ K $ times until (i) the node contains a single data point $(y_i, \bx^T_i)$ or (ii) all input values $ \bx_i $ and/or all response values $y_i$ within the node are the same.

In a conventional regression problem, where the goal is to estimate the conditional mean response $ \mu(\bx) $, the tree output for $ \bx \in \bt $ is the within-node sample mean $ \overline y_{\bt} $, i.e., if $ T $ is a decision tree, then  $ \hat \mu(T)(\bx) = \overline y_{\bt} = \frac{1}{\#\{\bx_i \in \bt\}}\sum_{\bx_i \in \bt}y_i $. However, one can aggregate the data in the node in a number of ways, depending on the target estimand. For example, CART methodology is also commonly used for classification tasks (e.g., propensity score estimation in causal inference settings), in particular, where the outcome variable $y_i\in\{0,1\}$ takes on binary values. In this case, the classification tree output is the majority vote of the class instances in the node. Because the canonical splitting criterion for binary classification, the \emph{Gini index}, is equivalent to \eqref{eq:se}, the results presented in this paper are directly applicable. In addition, decision tree methodology can also be employed for conditional quantile regression and its various downstream tasks, such as estimating quantiles, constructing confidence intervals, or performing outlier detection \cite[and references therein]{meinshausen2006quantile}. These methods also require high pointwise accuracy of decision trees, and thus our results will have methodological implications in those settings as well.

Furthermore, in multi-step semiparametric settings, it is often the case that preliminary unknown functions (e.g., propensity scores in causal inference settings) are estimated using modern machine learning methods such as CART \cite[see, for example,][and references therein]{chernozhukov2022locally}. Our results reveal that reliance on fast uniform convergence rates for decision tree methodology may not be guaranteed, as we show below that decision trees will have a convergence rate slower than any polynomial-in-$n$, over the entire support $\mathcal{X}$. This finding implies that other machine learning procedures such as neural networks \citep[and references therein]{farrell2021deep} may be preferable in those multi-step semiparametric settings, if such methods could be shown to be uniformly consistent with sufficiently fast rates of convergence.

From a big picture perspective, our main methodological message is to warn against mechanical application of flexible, adaptive machine learning methodologies for tasks that require good quality estimates at specific covariate values of interest. Machine learning procedures that are currently deployed in practice (for canonical regression problems) are trained to approximately minimize the empirical mean squared error. As such, they enjoy good out-of-sample accuracy for an average-case value of the covariates, i.e., if accuracy is measured via the integrated mean squared error (IMSE). However, if the task requires a more stringent form of convergence, such as uniform convergence, it is unknown if those procedures meet such additional demands. Our results are the first to formally show that this is not the case for decision trees, despite them having small IMSE.

\section{Causal Inference and Policy Decisions}\label{sec: Causal Inference and Policy Decisions}

As mentioned earlier, recursive partitioning is now a common tool of choice in the analysis of heterogeneous causal treatment effects and the design of heterogeneous policy interventions \citep[and references therein]{athey2019machine,yao2021survey}. Here the observed data is a random sample $ \{(y_i, \bx^T_i, d_i):i=1,2,\dots,n\}$, where $y_i$ is the outcome of interest, $\bx_i$ is a set of pre-treatment covariates, and $d_i$ is a binary treatment indicator variable. Employing standard potential outcomes notation,
\begin{equation}\label{eq:rubin}
y_i =  y_i{(1)} \cdot d_i + y_i{(0)} \cdot (1-d_i),
\end{equation}
where $y_i{(1)}$ is the potential outcome under treatment ($d_i=1$) and $y_i{(0)}$ is the potential outcome under control ($d_i=0$). This paradigm is fundamental to most applied sciences; for example, it can be used to model the effectiveness of a drug therapy, behavioral intervention, marketing campaign, 
or government program.

In cases where the individual treatment effect $ y_i{(1)} - y_i{(0)} $ varies across different subgroups, a natural goal is to estimate the \emph{heterogeneous} average treatment effect (ATE) for each covariate value $ \bx \in \mathcal{X}$, namely, $ \theta(\bx) = \mathbb{E}[y_i{(1)}-y_i{(0)} \mid \bx_i=\bx] $. In recent years, there has been an explosion of machine learning technologies adapted for heterogeneous causal effect estimation, thanks to the abundance of data produced from large-scale experiments and observational studies.  Among these machine learning algorithms, recursive partitioning estimators (specifically, \emph{causal decision trees}) stand out as natural contenders, as they are well-suited for grouping data according to the treatment effect size, conditional on observable characteristics \citep[e.g.,][]{su2009subgroup, athey2016recursive}.

We now discuss CART methodology in the context of heterogeneous causal effect estimation, one popular application of decision trees where accurate pointwise estimates over the entire support $\mathcal{X}$ are essential.
In experimental settings, where $ (y_i{(0)},y_i{(1)},\bx_i^T) \Perp d_i $, the \textit{conditional} ATE is identifiable because
\begin{align}
\theta(\bx_i) &= \mathbb{E}[y_i \mid \bx_i, \;d_i=1] - \mathbb{E}[y_i \mid \bx_i,\;d_i=0]\\
              &= \mathbb{E}\Bigg[y_i \frac{d_i-\xi}{\xi(1-\xi)} \mid \bx_i \Bigg],
\end{align}
where the probability of treatment assignment $\xi=\mathbb{P}(d_i=1)$ is known by virtue of the known randomization mechanism. It follows that $\theta(\bx)$, $\bx\in\mathcal{X}$, can be estimated using decision tree methodology in at least two ways, namely, for a decision tree $T$,
\begin{equation} \label{eq:imp}
\hat\theta_\text{reg}(T)(\bx) = \frac{1}{\#\{\bx_i \in \bt: d_i = 1\}}\sum_{\bx_i \in \bt: d_i = 1}y_i - \frac{1}{\#\{\bx_i \in \bt: d_i = 0\}}\sum_{\bx_i \in \bt: d_i = 0}y_i,
\end{equation}
or
\begin{equation} \label{eq:ipw}
\hat\theta_\text{ipw}(T)(\bx) = \frac{1}{\#\{\bx_i \in \bt\}}\sum_{\bx_i \in \bt} y_i \frac{d_i-\xi}{\xi(1-\xi)},
\end{equation}
where recall $\bt$ denotes the unique (terminal) node containing $\bx\in\mathcal{X}$.

In this spirit, we consider a tree-based approach for analyzing treatment effect heterogeneity in randomized control trials, which may also be used to design personalized treatment assignments based on pre-intervention observable characteristics. While our forthcoming results are stated for the regression problem \eqref{eq:model}, they are also directly applicable to the causal decision tree estimators above that involve minimizing the SSE criterion. This is precisely because $ \hat\theta_\text{reg}(T)(\bx) $ and $\hat\theta_\text{ipw}(T)(\bx)$ can be implemented using conventional CART methodology. That is, we implement $ \hat\theta_\text{reg}(T)(\bx) $ following a plug-in approach that estimates $ \mathbb{E}[y_i \mid \bx_i, \;d_i=1] $ and $ \mathbb{E}[y_i \mid \bx_i, \;d_i=0] $ separately with regression trees and conventional CART methodology. Alternatively, we fit a regression tree with CART methodology to the transformed outcome $ y_i(d_i-\xi)/(\xi(1-\xi)) $ to implement $\hat\theta_\text{ipw}(T)(\bx)$. Yet another (more principled) approach \citep{athey2016recursive} implements $\hat\theta_\text{reg}(T)(\bx) $ by growing a decision tree using a slightly modified version of the SSE criterion \eqref{eq:se} (referred to as \emph{adjusted expected MSE}) that more directly targets the conditional ATE, together with an \emph{honest} property, where different samples are used for constructing the partition and estimating the effects of each subpopulation.

Our theory implies that, for a constant treatment effect model, the aforementioned causal decision tree estimators cannot converge faster than any polynomial-in-$n$. Furthermore, in more interesting cases, shallow (honest) causal decision tree estimators will be shown to be inconsistent, as a function of the sample size $n$, for some $\bx\in\mathcal{X}$. Finally, we will also show that random forest methodology, while hurting the interpretability and introducing additional tuning parameters, can overcome the limitations of decision trees by restoring nearly optimal pointwise (for all $\bx\in\mathcal{X}$) convergence rates.

\section{Homoskedastic Constant Regression Model}

To formalize the pitfalls of pointwise regression estimation using decision trees, we consider the simplest possible data generating process.

\begin{assumption}[Location Regression Model]\label{ass:DGP}
    The observed data $\{(y_i,\bx_i^T) : i = 1, 2, \dots, n\}$ is a random sample satisfying \eqref{eq:model} and the following:
    \begin{enumerate}
        \item $ \mu(\bx) \equiv \mu $ is constant for all $\bx\in\mathcal{X}\subseteq\mathbb{R}^p$.
        \item $\bx_i$ has a continuous distribution.
        \item $ \bx_i \Perp \varepsilon_i$ for all $i=1,2,\dots,n$.
        \item $ \mathbb{E}\big[\varepsilon_i^2\log\log(|\varepsilon_i|+1)\big]<\infty$ and $ \mathbb{E}\big[\varepsilon_i^2\big] > 0 $.
    \end{enumerate}
\end{assumption}
Because trees are invariant with respect to monotone transformations of the coordinates of $ \bx $, without loss of generality, we assume henceforth that the marginal distributions of the covariates are uniformly distributed on $\mathcal{X} = [0, 1]^p$, i.e., $ x_j \thicksim U([0, 1])$ for $ j = 1, 2, \dots, p$.

Under Assumption \ref{ass:DGP}, the regression model \eqref{eq:model} becomes the standard location (or \emph{intercept-only regression}) model with homoskedastic errors:
\begin{equation} \label{eq:pure}
y_i = \mu + \varepsilon_i, \qquad \sigma^2=\mathbb{E}\big[\varepsilon_i^2\big].
\end{equation}
In the causal setting, the assumption corresponds to the constant treatment effect model, in which $ \theta(\bx) \equiv \theta $ is constant for all pre-treatment covariates $ \bx $.

This statistical model is perhaps the most canonical member of any interesting set of data generating processes. In particular, the regression function belongs to all classical smoothness function classes, as well as to the set of functions with bounded total variation. See, for example, \cite{gyorfi2002distribution} for review and further references. As a consequence, our results will also shed light in settings where uniformity over any of the aforementioned classes of functions is of interest, since our lower bounds can be applied directly in those cases. To be more precise, if $ \hat\mu(T)(\bx) $ is the output from a decision tree $ T $, then for any class of data generating processes $\mathcal{P}$ containing the model defined by Assumption \ref{ass:DGP}, $\sup_{\mathbb{P}\in\mathcal{P}}\mathbb{P}( \sup_{\bx \in \mathcal{X}}|\hat\mu(T)(\bx) - \mu(\bx)| > \epsilon) \geq \mathbb{P}(\sup_{\bx \in \mathcal{X}}|\hat\mu(T)(\bx) - \mu| > \epsilon)$, for any $ \epsilon > 0 $. Because $\mathcal{P}$ will include the model defined by Assumption \ref{ass:DGP} in all relevant (both theoretically and practically) cases, our results also highlight fundamental limitations of CART regression methods from a uniform (over $\mathcal{P}$) perspective, whenever interest lies on estimation of $\mu(\bx)$ for all $\bx\in\mathcal{X}$.

Since the main purpose of this paper is to explore the limits of decision tree methodology, we do not aim for generality, but rather consider the simplest possible data generating process (Assumption \ref{ass:DGP}). In the context of causal inference and treatment effects (e.g., Section \ref{sec: Causal Inference and Policy Decisions}), the assumptions correspond to a constant treatment effect model, the most basic case of practical interest. Importantly, Assumption \ref{ass:DGP} removes issues related to smoothing (or misspecification) bias because the regression function $\mu(\bx)$ is constant for all $\bx\in\mathcal{X}$, which shows that our results will not be driven by standard (boundary or other smoothing) bias in nonparametrics \citep{fan1996local}. Indeed, if the distribution of $ \varepsilon_i $ is symmetric, then we have
$
\mathbb{E}[\hat\mu(T)(\bx)-\mu] = - \mathbb{E}[\hat\mu(T)(\bx)-\mu] \implies \mathbb{E}[\hat\mu(T)(\bx)] = \mu,
$
owing to the fact that the split points $ \hat \tau $ are symmetric statistics of the $\varepsilon_i$. Our results will be driven instead by the fact that decision tree methodology can generate small cells containing only a handful of observations, thereby making the estimator imprecise in certain regions of $\mathcal{X}$. In other words, inconsistency is due to a \emph{large variance} problem, not a large bias problem. 

The location (or constant treatment effect) model is the simplest instantiation of a regression model of practical interest because the regression function is supersmooth and the curse of dimensionality is absent. Furthermore, all smooth regression functions can be seen as locally constant. Thus, we should expect any competitive nonparametric estimator to separate a constant signal from noise or, in the language of causal inference, to estimate accurately (constant) treatment effects when they happen to be homogeneous. Assumption \ref{ass:DGP} also approximately captures another common modeling situation in machine learning and data science, in which the marginal distribution $ y_i \mid x_{ij} $ is noisy (i.e., the marginal projections $ \mathbb{E}[y_i \mid x_{ij}] $ are constant and contain no signal). Because splits in trees are determined using only marginal information, here the split at the root node would be essentially fitting the location model.

\section{Decision Stumps} \label{sec:decision_stumps}

For each variable $ j = 1, 2, \dots, p $, the data $ \{ x_{ij}: \bx_i \in \bt \} $ is relabeled so that $ x_{ij} $ is increasing in the index $ i = 1, 2, \dots, n(\bt) $, where $ n(\bt) = \#\{ \bx_i \in \bt \} $.
Then, minimizing the objective \eqref{eq:se} can be equivalently recast as maximizing the so-called \emph{impurity gain}:
\begin{equation}
\label{eq:Delta}
\begin{aligned}
& \sum_{\bx_l \in \bt}\big(y_l-\overline y_{\bt}\big)^2 - \sum_{\bx_l \in \bt}\big(y_l - \overline y_{\bt_L} \Indicator(x_{lj} \leq \tau) - \overline y_{\bt_R} \Indicator(x_{lj} > \tau)\big)^2
 \\
 & \qquad = \frac{\Big(\frac{1}{\sqrt{n(\bt)}}\sum_{l=1}^{i} (y_l-\mu) - \frac{i}{n(\bt)}\frac{1}{\sqrt{n(\bt)}}\sum_{l=1}^{n(\bt)} (y_l-\mu) \Big)^2}{i(n(\bt)-i)},
 \end{aligned}
\end{equation}
with respect to the index $ i $ and variable $ j $; see \citep{breiman1984}. The maximizers are denoted by $ (\hat\imath, \hat\jmath) $, and the optimal split point $ \hat \tau $ that minimizes \eqref{eq:se} can be expressed as $ x_{\hat\imath\hat\jmath} $.

We start by considering the case when the tree is depth one ($K=1$), i.e., a decision stump. The tree output can then be written as
\begin{equation} \label{eq:stump}
    \hat\mu(T_1)(\bx) = \hat\beta_1\Indicator(x_{\hat\jmath} \leq \hat \tau) + \hat\beta_2\Indicator(x_{\hat\jmath} > \hat \tau)= \begin{cases}
     \frac{1}{\#\{\bx_i: x_{i \hat\jmath} \leq x_{\hat\imath\hat\jmath}\}} \sum_{\bx_i:x_{i \hat\jmath} \leq x_{\hat\imath\hat\jmath}} y_i, \quad x_{\hat\jmath} \leq x_{\hat\imath\hat\jmath} \\
      \frac{1}{\#\{\bx_i: x_{i \hat\jmath} > x_{\hat\imath\hat\jmath}\}} \sum_{\bx_i:x_{i \hat\jmath} > x_{\hat\imath\hat\jmath}} y_i, \quad x_{\hat\jmath} > x_{\hat\imath\hat\jmath}
    \end{cases}\hspace{-3mm},
\end{equation}
where $x_{\hat\jmath}$ denotes the value of the $\hat\jmath$-th component of $\bx$.

The following theorem formally (and very precisely) characterizes the regions of the support $\mathcal{X}$ where the first CART split index $ \hat\imath $, at the root node, has non-vanishing probability of realizing. As a consequence, the theorem also characterizes the effective sample size of the resulting cells (recall the data is ordered so that $ \hat \tau = x_{\hat\imath \hat\jmath} $ and hence $ \hat\imath = \#\{\bx_i : x_{ij} \leq \hat \tau \} $).

\begin{theorem}\label{thm:master}
   Suppose Assumption \ref{ass:DGP} holds and $ p = 1 $, and let $ \hat\imath $ be the CART split index at the root node.
    For each $ a,b \in (0, 1) $ with $ a < b $, we have
        \begin{equation} \label{eq:split_range}
    \liminf_{n\to\infty} \mathbb{P}\big( n^{a} \leq \hat\imath \leq n^{b}\big) = \liminf_{n\to\infty} \mathbb{P}\big( n-n^{b} \leq \hat\imath \leq n-n^{a} \big) \geq \frac{b-a}{e}.
    \end{equation}
\end{theorem}

\begin{remark}
We conjecture that $ \lim_{n\to\infty} \mathbb{P}\big( n^{a} \leq \hat\imath \leq n^{b}\big) = \lim_{n\to\infty} \mathbb{P}\big( n-n^{b} \leq \hat\imath \leq n-n^{a} \big) = (b-a)/2 $. Another way of stating this conjecture is that the asymptotic conditional distribution of $\log(\hat\imath)/\log(n) $ given that $ \hat\imath \leq n/2 $ is $U([0, 1]) $. Evidence for this limit will be given in the proof of Theorem \ref{thm:master}.
\end{remark}

First, Theorem \ref{thm:master} shows that with non-vanishing probability, $ \hat\imath $ will realize near its extremes, from the beginning of any tree construction. The arbitrarily slow polynomial-in-$n$ rates do not contradict, but are rather precluded by, existing polynomial convergence guarantees \citep[e.g.,][]{wager2018estimation}, which a priori require that each split generates two child nodes that contain a constant fraction of the number of observations in the parent node, i.e., $ n(\bt_L) \gtrsim n(\bt) $ and $ n(\bt_R) \gtrsim n(\bt) $. By implication, Theorem \ref{thm:master} shows that such assumptions requiring \emph{balanced} cells almost surely, which are typically imposed in the literature, are in general incompatible with standard decision tree constructions employing conventional CART methodology \citep[e.g.,][and references therein]{behr2022provable}. The slow convergence rates for the decision stump occur because the optimal split point is realized near the boundary of the support \citep{ishwaran2015effect} with non-vanishing probability, i.e., $ \hat \tau \approx 0 $ or $ \hat \tau \approx 1 $ with non-vanishing probability, causing the two nodes in the stump to be imbalanced, with one containing a much smaller number of samples, and therefore rendering a situation where local averaging is less accurate. To be more precise, after the first split when $ n(\bt) = n $, CART will generate two unbalanced cells with non-vanishing probability; for any $ a,b \in (0, 1)$ with $ a < b $, either $ n^{a} \leq n(\bt_L) \leq n^{b} $ or $ n^{a} \leq n(\bt_R) \leq n^{b} $ for large $ n $, where $ n(\bt_L) + n(\bt_R) = n $. It will follow from this result that, on the events considered in \eqref{eq:split_range}, too few observations will be available on one of the cells after the first split for CART to deliver a polynomial-in-$n$ consistent estimator of $\mu$, thereby making the decision tree procedure exhibit arbitrarily slow rates, for some $x \in \mathcal{X} $.

\subsection{Convergence Rates}

Theorem \ref{thm:master} appears to be new in the literature. It arises from a careful study of the maximum of \eqref{eq:Delta} over different ranges of the split index, which turns out to be asymptotically similar to the suprema of a standardized Brownian bridge over different time intervals.

Once the location of the first CART split point is well-understood, we can study the resulting CART estimator $\hat\mu(T_1)(\bx)$ of the unknown regression function. The following statements hold for the pointwise prediction error of the decision stump.

\begin{theorem}\label{thm:rates}
    Suppose Assumption \ref{ass:DGP} holds and $ p = 1 $, and let $\hat\mu(T_1)(x)$ be the CART estimator of the regression function at the root node. For any $ a,b \in (0, 1) $ with $ a < b $, we have
        \begin{equation}\label{eq:rate_constant2}
    \liminf_{n\to\infty} \mathbb{P}\Bigg(\sup_{x\in\mathcal{X}}|\hat\mu(T_1)(x) - \mu| \geq \sigma n^{-b/2}\sqrt{(2+o(1))\log\log(n)}\Bigg) \geq \frac{2b}{e},
    \end{equation}
    and
    \begin{equation} \label{eq:rate_constant}
    \liminf_{n\to\infty} \inf_{x\in \mathcal{X}_n} \mathbb{P}\Big(|\hat\mu(T_1)(x) - \mu| \geq \sigma n^{-b/2}\sqrt{(2+o(1))\log\log(n)}\Big) \geq \frac{b-a}{e},
    \end{equation}
    where $ \mathcal{X}_n = [0,\, (1+o(1))n^{a-1}) \cup (1- (1+o(1))n^{a-1},\, 1]$.
\end{theorem}

The theorem above shows that decision stumps can have, at most, $n^{b/2}$ (suboptimal) convergence for evaluation points that are within $n^{a-1}$ distance from the boundary of $\mathcal{X}$ (see \eqref{eq:rate_constant}), for \emph{any} $ a, b\in(0,1)$ with $ a < b $. This happens because the two nodes in the stump are highly imbalanced with non-trivial probability under Assumption \ref{ass:DGP}, with one containing a much smaller number of samples—thereby making local estimation difficult. An immediate implication of Theorem \ref{thm:rates} in the context of heterogeneous (in $\bx\in\mathcal{X}$) causal effect estimation is that the CART estimators discussed in Section \ref{sec: Causal Inference and Policy Decisions} can have poor performance in some regions of the covariate support, particularly near the boundaries of $\mathcal{X}$. 

\subsection{Past Work}

Theorem \ref{thm:rates} contributes to the literature in several ways. Our results indicate that when the goal is to approximate the unknown conditional expectation pointwise for all $\bx\in\mathcal{X}$, as it is the case in the analysis of heterogeneity in causal inference settings, decision trees will exhibit extremely slow convergence rates in some regions of the support, making those methods suboptimal from an approximation perspective. The phenomenon revealed in Theorems \ref{thm:master} and \ref{thm:rates} has been observed in various forms since the inception of CART \cite[Section 11]{breiman1984}. Historically, the phenomenon characterized in Theorem \ref{thm:master} has been called the \emph{end-cut preference}, where splits along noisy directions tend to concentrate along the end points of the parent node. More specifically, \citet[Theorem 11.1]{breiman1984} and \citet[Theorem 4]{ishwaran2015effect} showed that for each $ \delta \in (0, 1) $, $ \mathbb{P}( \hat\imath \leq \delta n \; \text{or} \;\hat\imath \geq (1-\delta)n) \rightarrow 1 $ as $n\rightarrow \infty $. However, unlike \eqref{eq:rate_constant} in Theorem \ref{thm:master} which characterizes regions of the support where the pointwise rates of estimation are slower than any \emph{polynomial-in-$n$}, their result only implies rates in uniform norm slower than any \emph{constant multiple} of the already nearly optimal rate $\sqrt{n/\log\log(n)}$, i.e., for any $ C > 0 $,
$$\lim_{n\to\infty} \mathbb{P}\Big(\sup_{x\in\mathcal{X}}|\hat\mu(T_1)(x) - \mu| \geq C\sigma n^{-1/2}\sqrt{\log\log(n)}\Big) =1.$$
Thus, past theoretical work is not strong enough to illustrate the weaknesses of decision trees for pointwise estimation (i.e., prior lower bounds in the literature were too loose to be informative).

In accordance with Theorem \ref{thm:rates}, simulation results from \citet[Supplement, Section B]{wager2018estimation}, and many others, also suggested that adaptive causal trees can have slow convergence at the boundaries of the support $\mathcal{X}$, but no formal theory supporting that numerical evidence was available in the literature until now. \citet{tang2018when} give sufficient theoretical conditions under which non-adaptive random forests (i.e., where the decision nodes are independent of the data) will be inconsistent, but those conditions do not apply to commonly used forest implementations nor are they shown to be realized by the data generating mechanism.

\cite{buhlmann2002analyzing} and \cite{banerjee2007confidence} showed that the minimizers $(\hat\beta_1, \hat\beta_2, \hat \tau)$ of \eqref{eq:se} at the root node converge to the population minimizers $ (\beta^*_1, \beta^*_2, \tau^*) $ at a cube-root $ n^{1/3} $ rate when the regression model \eqref{eq:model} satisfies specific regularity assumptions. Because the decision stump \eqref{eq:stump} can be expressed as $ \hat\mu(T_1)(x) = \hat\beta_1\Indicator(x \leq \hat \tau) + \hat\beta_2\Indicator(x > \hat \tau) $, their results can be used to study the asymptotic properties of $\hat\mu(T_1)(x)$. Among other things, they posit that the population minimizers $ (\beta^*_1, \beta^*_2, \tau^*) $ are unique and that the regression function $ \mu(x) $ is continuously differentiable and has nonzero derivative at $ \tau^* $. Theorem \ref{thm:rates} shows that the results in \cite{buhlmann2002analyzing} and \cite{banerjee2007confidence} are not uniformly valid in the sense that excluding the constant regression function from the allowed class of data generating processes is necessary for their results to hold for $x\in\mathcal{X}$.

\subsection{Uniform Minimax Rates}

Letting $\mathcal{P}$ be any set of data generating processes of interest that includes the location model in Assumption \ref{ass:DGP},  for any $b \in (0, 1) $, we immediately obtain from \eqref{eq:rate_constant2} that
$$\liminf_{n\to\infty}\sup_{\mathbb{P}\in\mathcal{P}}\mathbb{P}\Bigg(\sup_{x\in\mathcal{X}}|\hat\mu(T)(x) - \mu| \geq \sigma n^{-b/2}\sqrt{(2+o(1))\log\log(n)}\Bigg) \geq (2/e)b,$$
where $ T $ is any tree constructed using conventional CART methodology with at least one split.
Therefore, decision trees grown with CART methodology cannot converge faster than any polynomial-in-$n$, when uniformity over the full support of the data $\mathcal{X}$, and over possible data generating processes, is of interest. 

\subsection{Honest Trees}

While Theorem \ref{thm:rates} deals with depth $K=1$ adaptive trees (i.e., the same data is used for determining the split points and terminal node output), analogous results hold for honest trees. The honest tree output is
\begin{equation} \label{eq:honest_tree}
\tilde\mu(T)(\bx) = \frac{1}{\#\{\tilde \bx_i \in \bt\}}\sum_{\tilde \bx_i \in \bt} \tilde y_i, \quad \bx \in \bt,
\end{equation}
where $( \tilde y_i, \tilde \bx_i^T) $, $i=1,2,\dots,n$, are independent samples from those which were used to construct the decision nodes (i.e., the partition of $ \mathcal{X} $), and $ n(\bt) = \#\{\tilde \bx_i \in \bt\} > 0 $. To simplify calculations, we define $ \tilde\mu(T)(\bx) = \mu(\bx) $ if $ n(\bt) = 0 $, an event that occurs with vanishingly small probability.

Conditional on the data used to construct the partition, the honest decision stump $ \tilde \mu(T_1)(x) $ at $ x = 0 $ is an average of (approximately) $ \hat\imath $ response values, and so we expect its variance (equal to mean squared error) to be approximately $ \sigma^2/\hat\imath $. The problem is that, according to Theorem \ref{thm:master}, the split index $ \hat\imath $ is much smaller than $ n $, with non-vanishing probability. More rigorously, using a conditioning argument 
and \eqref{eq:split_range}, it follows that $ \tilde\mu(T_1)(x) $ converges uniformly no faster than
\begin{equation} \label{eq:honest_lower}
\mathbb{E}\Bigg[\sup_{x\in\mathcal{X}}\,\big(\tilde\mu(T_1)(x)-\mu\big)^2\Bigg] \geq \sigma^2\mathbb{E}\Bigg[\frac{(1-2^{-\hat\imath})^2}{\hat\imath}\Bigg] \gtrsim \frac{\sigma^2}{n^b},
\end{equation}
for any $ b \in (0, 1) $, and $n$ large enough.

\subsection{Simulation Evidence} \label{sec:numerical_stump}

We illustrate the implications of Theorems \ref{thm:master} and \ref{thm:rates} numerically with $p=1$. In Figure \ref{fig:boundary}, we plot the pointwise root mean squared error (RMSE) $ \sqrt{\mathbb{E}[(\hat\mu(T_1)(x)-\mu)^2]} $, approximated by $500$ replications, when $ \mu = 0 $, $ \varepsilon_i \thicksim N(0, 1) $, and $ n = 1000 $. In Figure \ref{fig:boundary_causal}, we consider the context of the causal model discussed in Section \ref{sec: Causal Inference and Policy Decisions}, with a constant treatment effect $\theta(x) = 1 $ and $ \mathbb{E}[y_i(0)]= 0 $, $ d_i \thicksim \text{Bern}(0.5) $, and $ \varepsilon_i \thicksim N(0, 1) $, again with $ n = 1000 $ and $ 500 $ replications. We plot the pointwise RMSE for an honest causal decision stump with output based on the regression estimator $\hat\theta_\text{reg}(T_1)(x)$ constructed using the adjusted expected MSE splitting criterion proposed by \cite{athey2016recursive}. The transformed outcome tree, $\hat\theta_\text{ipw}(T_1)(x) $, exhibits similar empirical behavior. Both plots corroborate with Theorem \ref{thm:rates}: the decision stump has smallest pointwise RMSE near the center of the covariate space, but the performance degrades as the evaluation points move closer to the boundary.

\begin{figure}
\centering
    \begin{subfigure}[t]{0.4\textwidth}
         \centering
        \includegraphics[width=1\textwidth]{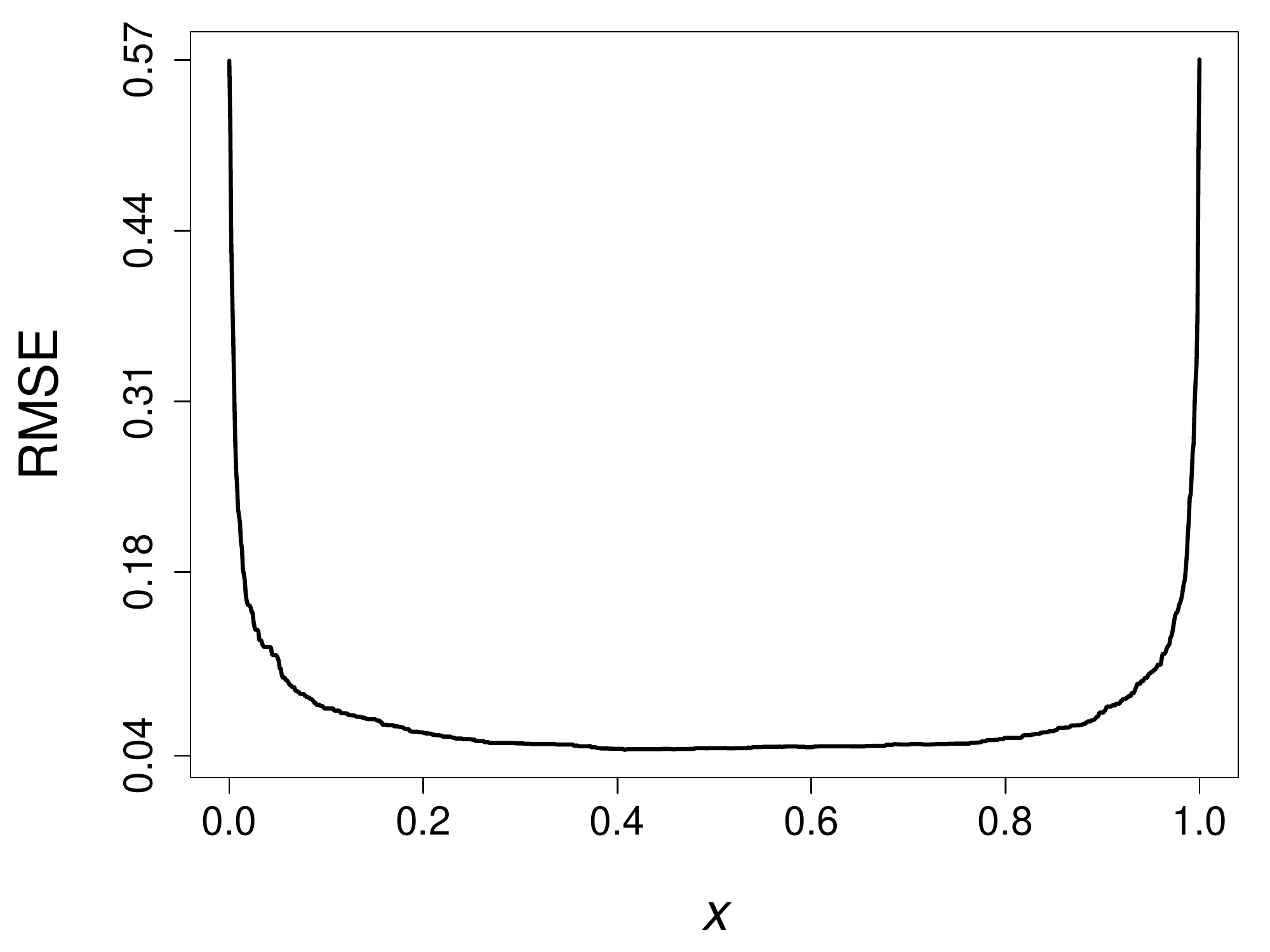}
    \caption{Pointwise RMSE of decision stump.}
          \label{fig:boundary}
\end{subfigure}
\hspace{1.5cm}
\begin{subfigure}[t]{0.4\textwidth}
    \centering
    \includegraphics[width=1\textwidth]{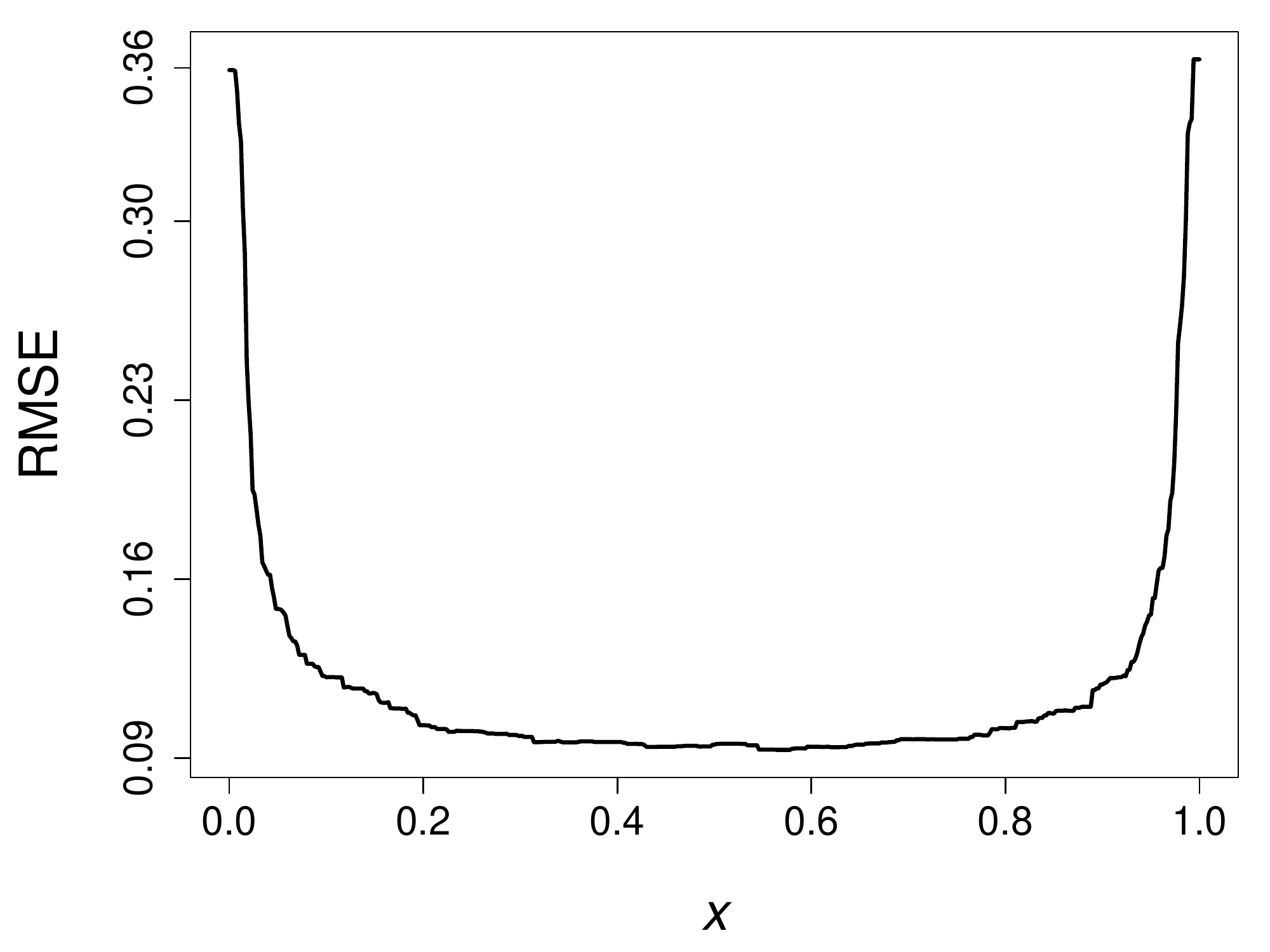}
    \caption{Pointwise RMSE of causal decision stump.}
          \label{fig:boundary_causal}
          \end{subfigure}                   
          \caption{Pointwise RMSE of decision stumps for location model.}
         \label{fig:stump}
\end{figure}

The following section investigates further the role of honesty in the construction of deeper trees, and shows an even stronger result: honest trees will be inconsistent on some (at least countably many) regions of $\mathcal{X}$ whenever the trees are grown up to depth $K\approx \log\log(n)$. In other words, shallow (honest) regression trees can be uniformly inconsistent, a result that is intuitively anticipated from Theorems \ref{thm:master} and \ref{thm:rates} because even after one single split there is non-trivial probability of having small cells with only a few observations, and repeating this process further down the tree can only exacerbate the issue.

The main results in this section were derived in the simplest possible case (constant regression model, $p=1$, $K=1$, etc.), but the main conclusions are applicable more generally. The key phenomenon captured by Theorems \ref{thm:master} and \ref{thm:rates} are only exacerbated in multi-dimensional settings ($p>1$) or for multi-level decision trees ($K>1$). 
%We already demonstrated this fact for multi-dimensional covariates in \eqref{eq:rate_p}, and 
We will formalize the shortcomings associated with deeper honest trees in the next section. 

\section{Inconsistency with Deeper Trees}

The previous section provides a pessimistic view on depth one ($K=1$) decision trees: decision stumps can have slow convergence for the simplest regression models in some regions of $\mathcal{X}$. We now discuss formally situations where decision trees can be \emph{inconsistent }(i.e., fail to converge) altogether, if grown only to depth $K\approx \log\log(n)$. As is customary in the literature, we will focus on trees that are honest, which are believed to offer better empirical performance \citep{athey2016recursive}. 

\begin{definition}[Honest CART (CART+)]
At each level of the tree, including the output in the terminal nodes, generate new response values $ \{ \tilde y_i : i = 1, 2, \dots, n\} $. Each node $ \bt $ from the parent level is further refined by selecting a split direction and split point that minimizes the CART squared error criterion \eqref{eq:se} with data
$
\{ (\tilde y_i, \bx_i^T) :  \bx_i \in \bt \}.
$
The output of the tree $ T $ at a point $ \bx $ belonging to a terminal node $ \bt $ is $ \tilde\mu(T)(\bx) = \frac{1}{\#\{ \bx_i \in \bt \}}\sum_{ \bx_i \in \bt} \tilde y_i $.
\end{definition}

The only difference between conventional CART and CART+ is that the split points at each level are determined using a new, statistically independent set of response values, although the input values remain the same. Importantly, the adaptive properties of the tree are retained, as the nodes are still refined by minimizing the empirical squared error \eqref{eq:se}. 
 For our purposes, problems will arise as soon as the depth $ K $ is approximately $ \log\log(n) $ and so we expect there to be little practical difference between CART+ and the original CART algorithm when the sample size is large.

CART+ serves as a phenomenological model of conventional CART and allows us to analyze its pointwise (and uniform in $\bx\in\mathcal{X}$) behavior. Importantly, the formulation of CART+ and Assumption \ref{ass:DGP} together ensure that the split points have a desirable Markovian property: a split point $ \tilde \tau \in [\tau_1, \tau_2] $ conditioned on its immediate ancestor split points $ \tilde \tau_1 = \tau_1 $ and $ \tilde \tau_2 = \tau_2 $ is independent of all ancestor split points, including $ \tilde \tau_1 $ and $ \tilde \tau_2 $.

\begin{theorem} \label{thm:main}
Suppose Assumption \ref{ass:DGP} holds and $ p = 1 $. Consider a maximal depth $ K_n \gtrsim \log\log(n)$ tree $T_{K_n}$ constructed with CART+ methodology. Then, there exists a positive constant $C$ such that
	$$
	\liminf_{n\to\infty}\mathbb{P}\Bigg(\sup_{x\in\mathcal{X}}|\tilde\mu(T_{K_n})(x)-\mu| > C \Bigg) > 0.
	$$
\end{theorem}
This theorem shows that very shallow trees grown with the conventional squared error criterion can be pointwise (and hence uniform in $\bx\in\mathcal{X}$) inconsistent. To put the iterated logarithm scaling of the depth $K$ into perspective, if $ n = 1\; \text{billion} $, then $ \log\log(n) \approx 3 $, a typical depth seen in practice.

The pointwise error in Theorem \ref{thm:main} should be contrasted with the IMSE. Under Assumption \ref{ass:DGP}, if $ K \asymp \log\log(n) $, then
\begin{equation} \label{eq:imse_rate}
\mathbb{E}\Bigg[\int_{\mathcal{X}}(\tilde \mu(T_K)(x)-\mu)^2 \mathbb{P}_x(dx)\Bigg] \leq \frac{2^{K+1}\sigma^2}{n+1} = O\Bigg(\frac{\sigma^2\text{poly-log}(n)}{n}\Bigg),
\end{equation}
and hence by Markov's inequality,
$$
\lim_{n\to\infty} \mathbb{P}\Bigg( \int_{\mathcal{X}}(\tilde \mu(T_K)(x)-\mu)^2 \mathbb{P}_x(dx) > \frac{\sigma^2/n}{\text{poly-log}(n)} \Bigg) = 0.
$$
Therefore, the IMSE of the pointwise inconsistent depth $ K \asymp \log\log(n) $ decision tree decays at the optimal $\sqrt{n}$ rate, up to poly-logarithmic factors. This shows that the performance of the tree varies widely depending on whether the input $ x $ is average or worst case.

The intuition for Theorem \ref{thm:main} is based similarly on Theorem \ref{thm:master}, but for depth $K$ trees constructed with CART+ methodology. That is, honest trees of depth only $K \approx \log\log(n)$ will generate cells near the boundaries of the support $\mathcal{X}$ containing a finite number of observations with probability bounded away from zero. The inequality \eqref{eq:split_range} implies that, with probability bounded away from zero, the number of observations in a child node $ \bt' $ of a parent node $ \bt $ (near the boundary of $ \mathcal{X} $) satisfies $ n(\bt') \leq (n(\bt))^{b} $. It turns out that the number of times this occurs after $ K $ splits is stochastically dominated by a negative binomial random variable, providing a lower bound on the probability that a maximal depth $ K $ tree will have, at most, $ n(\bt) \leq n^{b^K} $ observations in terminal nodes near the boundary of $ \mathcal{X} $. Since $ b \in (0,1) $, the bound $ n^{b^K} $ is a constant whenever $K$ exceeds a constant multiple of $ \log\log(n)$.

It is important to note that the aforementioned inconsistency of honest regression trees need not occur at the boundary of the support $\mathcal{X}$. By a symmetry argument, if $ \tilde \tau $ is any split point that occurs at a fixed depth in the tree, then $ \tilde\mu(T_K)(\tilde \tau) $ will also fail to converge to $ \mu $ if the tree has maximal depth $ K \gtrsim \log\log(n) $. In other words, after reaching depth $ J \geq 1 $, inconsistency will occur at any of the (at most) $2^{J}+1$ endpoints associated with the $2^J$ cells, whenever we grow the tree to a total depth of $ J+K $ such that $ K \gtrsim \log\log(n)$ as $ n\rightarrow \infty$. 

\section{Pruning} \label{sec:pruning}

Pruning is a well-established strategy for mitigating some of the ill consequences of working with trees, such as overfitting. In some cases, however, pruning will not help. Indeed, as the previous section has revealed, depth one trees can have extremely slow convergence near the boundary of the covariate space. While this phenomenon holds for location models, it can also manifest with models that have a strong dependence on the covariates. For example, if the first split at the root node is along a variable $ x_j $ such that the marginal projection $ \mathbb{E}[y \mid x_j ] $ is constant---resembling the location model in Assumption \ref{ass:DGP} marginally---then, according to the previous discussion, the tree will almost always produce one cell with very few observations, but no amount of pruning at lower depths will help. The \emph{checkerboard} model \citep{bengio2010} in $ p = 2$ dimensions is an example where $y$ is marginally independent of both covariates. That is, if $ y_i = \text{sgn}( x_{i1}-0.5)\text{sgn}(x_{i2}-0.5) + \varepsilon_i $, where $ \bx_i \thicksim U([0,1]^2) $ and $ \varepsilon_i \thicksim N(0, 1) $ are independent, then $ y_i $ given $ x_{ij} = x_j $ is distributed as a symmetric two-component Gaussian mixture, free from $ x_j $.

To illustrate the point above numerically on a model with a smooth regression function, suppose $ y_i = (x_{i1}-0.5)(x_{i2}-0.5) + \varepsilon_i $, where $ \bx_i \thicksim U([0,1]^2) $ and $ \varepsilon_i \thicksim N(0, 1) $ are independent. As $ \mathbb{E}[y_i  \mid x_{ij} = x_j] = 0 $ for $ j = 1, 2$, the response variable has no marginal dependence on either covariate. Figure \ref{fig:boundary_prune} displays the results of a computer experiment with $ n = 1000 $ and $500$ replications. The plot shows the pointwise RMSE of a pruned tree $ T$ with output $ \hat\mu(T)(\bx) $ at $ \bx = (0, x_2)$ as $ x_2 $ ranges from $0$ to $1$. Similarly, Figure \ref{fig:boundary_prune_causal} shows the result of fitting a pruned causal tree $T$ with output $\hat\theta_\text{reg}(T)(\bx)$, constructed using honesty and the adjusted expected MSE splitting criterion proposed by \cite{athey2016recursive}. The experiment consists of $500$ replications from the model $ y_i = d_i(x_{i1}-0.5)(x_{i2}-0.5) + \varepsilon_i $, where $ d_i \thicksim \text{Bin}(0.5) $, $ \bx_i \thicksim U([0,1]^2) $, and $ \varepsilon_i \thicksim N(0, 1) $ are independent, and $n = 1000 $.  We do not include the transformed outcome tree $\hat\theta_\text{ipw}(T)(\bx) $ as it also produces a similar plot. In both cases, the numerical evidence indicates that pruning does not mitigate the lack of uniform consistency over $\mathcal{X}$ and the poor performance near the boundary persists.

\begin{figure}
\centering
    \begin{subfigure}[t]{0.4\textwidth}
         \centering
        \includegraphics[width=1\textwidth]{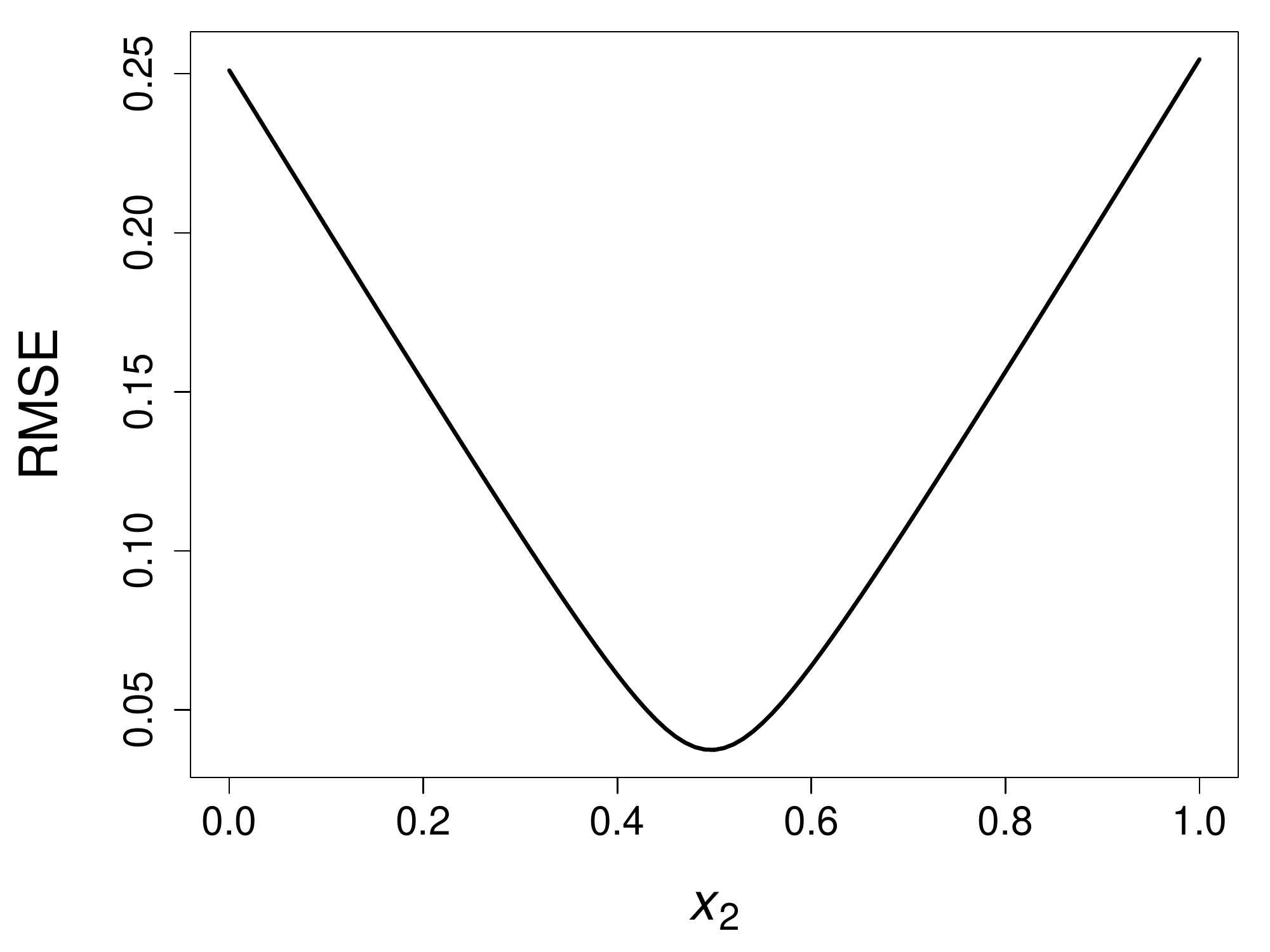}
    \caption{Pointwise RMSE for pruned tree at $ \bx = (0, x_2)^T$.}
          \label{fig:boundary_prune}
\end{subfigure}
\hspace{1.5cm}
\begin{subfigure}[t]{0.4\textwidth}
    \centering
    \includegraphics[width=1\textwidth]{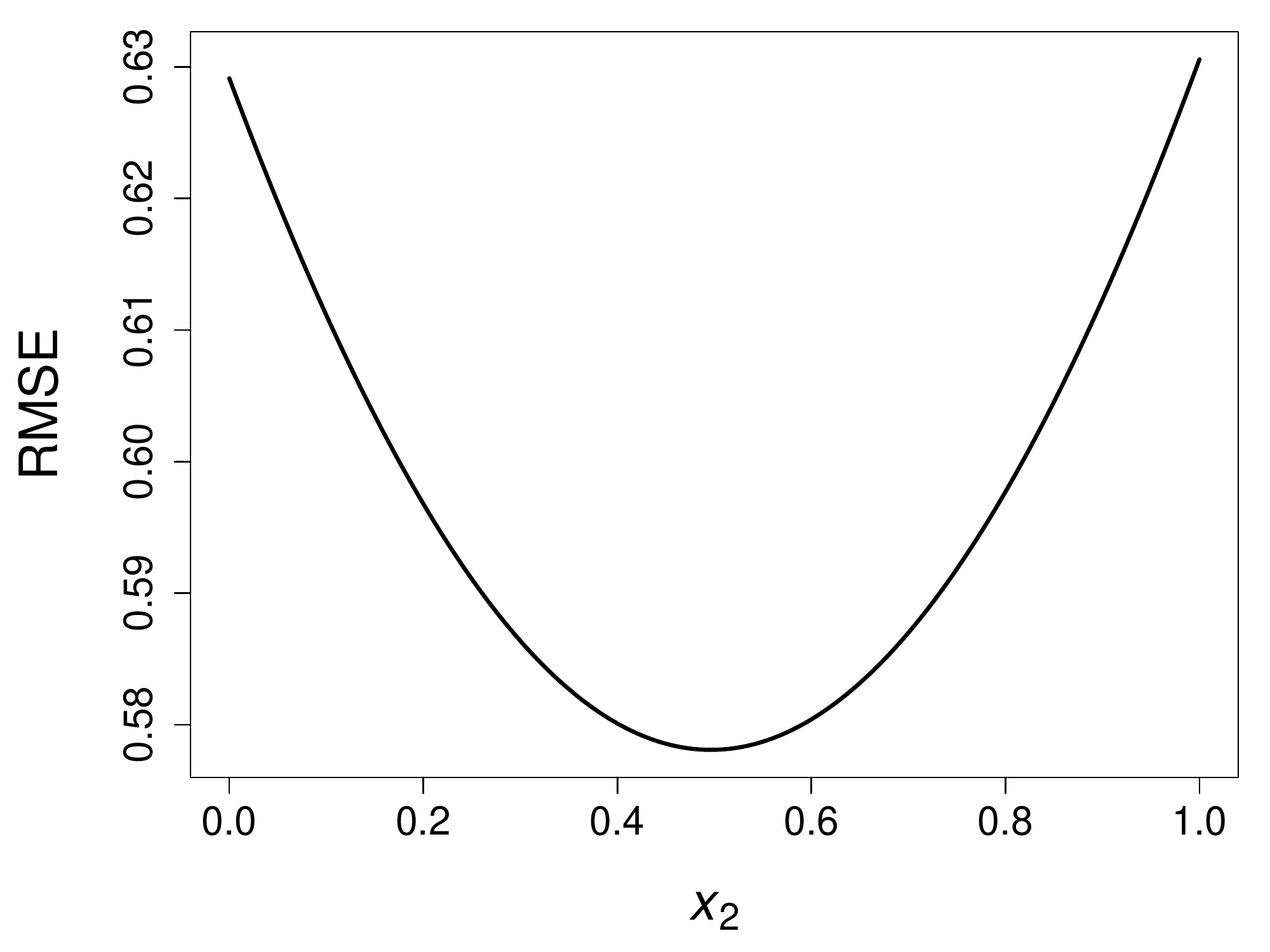}
    \caption{Pointwise RMSE for pruned causal tree at $ \bx = (0, x_2)^T$.}
          \label{fig:boundary_prune_causal}
          \end{subfigure}
          \caption{Pointwise RMSE of pruned trees for models where $ \bx $ and $ y $ are dependent.}
\end{figure}

\section{Random Forests}

At this point, the curious reader may wonder whether ensemble learning can address some of the convergence issues with decision trees. Here we consider \emph{honest random forests}, developed by \cite{wager2018estimation}. Specifically, for each tree in the ensemble, we randomly sample a subset $ S \subset \{ 1, 2, \dots, n\} $ of size $ s $ and, among the data $ \{(\bx_i, y_i)\}_{ i \in S} $, use half for determining the splits and the other half for estimating the conditional mean in the terminal nodes (the division of $ S $ into two equally sized subsets occurs randomly). More specifically, for each $ S \subset \{1, 2, \dots, n \} $ with $ |S| = s $, let $ S_0 $ denote the portion used for determining the splits and $ S_1 $ be the portion used for estimating the conditional mean in the terminal nodes. The set of all such subsamples is denoted by $ \mathcal{S} = \{S= S_1\cup S_0 \subset \{1, 2, \dots, n\}: S_0\cap S_1 = \emptyset, \; |S_0|=|S_1|=s/2\} $. In addition, at each node, a particular variable is split if it yields the smallest SSE \eqref{eq:se} among a random selection $ M \subset \{1, 2, \dots, p \} $ of $ m = mtry $ candidate directions. The set of all candidate variable selections is denoted by $ \mathcal{M} = \{M \subset \{1, 2, \dots p\} : |M| = m\}$. This idea can be applied to regression trees to obtain a regression forest, or causal decision tree estimators \eqref{eq:imp} or \eqref{eq:ipw} to obtain a causal forest, though, for simplicity, here we only consider the regression setting.

To get a sense of the improvement that forests offer over trees, we specialize to the case where the constituent trees in the forest are honest decision stumps (i.e., honest trees \eqref{eq:honest_tree} with depth $K=1$). The decision stump output $\hat\mu(T_1)(\bx) $ constructed in this way is denoted by $ \hat\mu(T(M, S))(\bx) $ and the (regression) random forest output is $ \hat\mu_B(\bx) = B^{-1}\sum_{b=1}^B \hat\mu(T(M_b, S_b))(\bx) $, where $ (M_1, S_1), (M_2, S_2), \dots, (M_B, S_B) $ are independent copies of $ (M, S) $.
When the number of trees $ B$ is large, the honest random forest can be approximated by
$$
\hat\mu(\bx) = \frac{1}{\binom{n}{s}\binom{s}{s/2}\binom{p}{m}} \sum_{S \in\mathcal{S}}\sum_{M \in\mathcal{M}}\hat\mu(T(M, S))(\bx).
$$

The next theorem provides an upper bound on its pointwise error.

\begin{theorem}\label{thm:forests}
Suppose Assumption \ref{ass:DGP} holds, and, additionally, that $ x_{i1}, x_{i2}, \dots, x_{ip} $ are independent. If $n \to \infty$, $p\to \infty$, $s = o(n^{1/3})$, and $ m = o(p/s) $, then for all $ \bx \in \mathcal{X} $,
$$\mathbb{E}\big[(\hat \mu(\bx) - \mu)^2\big]
\leq (\sigma^2/n)(1+(s/2)(m/p)+o(1)).
$$
\end{theorem}

This theorem showcases explicitly the effect of both subsampling and the random variable selection mechanism---each is important for reducing variance. According to past work that utilizes the Hoeffding-Serfling variance inequality for U-statistics \citep{wager2018estimation, buhlmann2002analyzing}, subsampling allows us to achieve a pointwise error $$ \mathbb{E}\big[(\hat\mu(\bx) - \mu)^2\big] \lesssim \sigma^2s/n, $$
which is significantly better than the arbitrarily slow polynomial-in-$n$ rates for individual trees (see Theorem \ref{thm:rates}), but still suboptimal since $ s $ is typically chosen to grow with the sample size to reduce bias when it exists. The result becomes more interesting when we account for the random variable selection mechanism, because it further reduces the error by decorrelating the constituent trees. Therefore, if the dimensionality $ p $ is large relative to $ s $ and $ m = o(p/s) $, then it is possible to achieve the \emph{exact} optimal $\sqrt{n}$ rate---a vast improvement over the $ n^{b/2} $ rate for individual trees. 
The price paid for such improvement is the inclusion of two additional tuning parameters for implementation ($s$ and $m$), and the loss of interpretability for the resulting estimates. 

\begin{figure}
   \centering
\begin{subfigure}[t]{0.4\textwidth}
    \centering
    \includegraphics[width=1\textwidth]{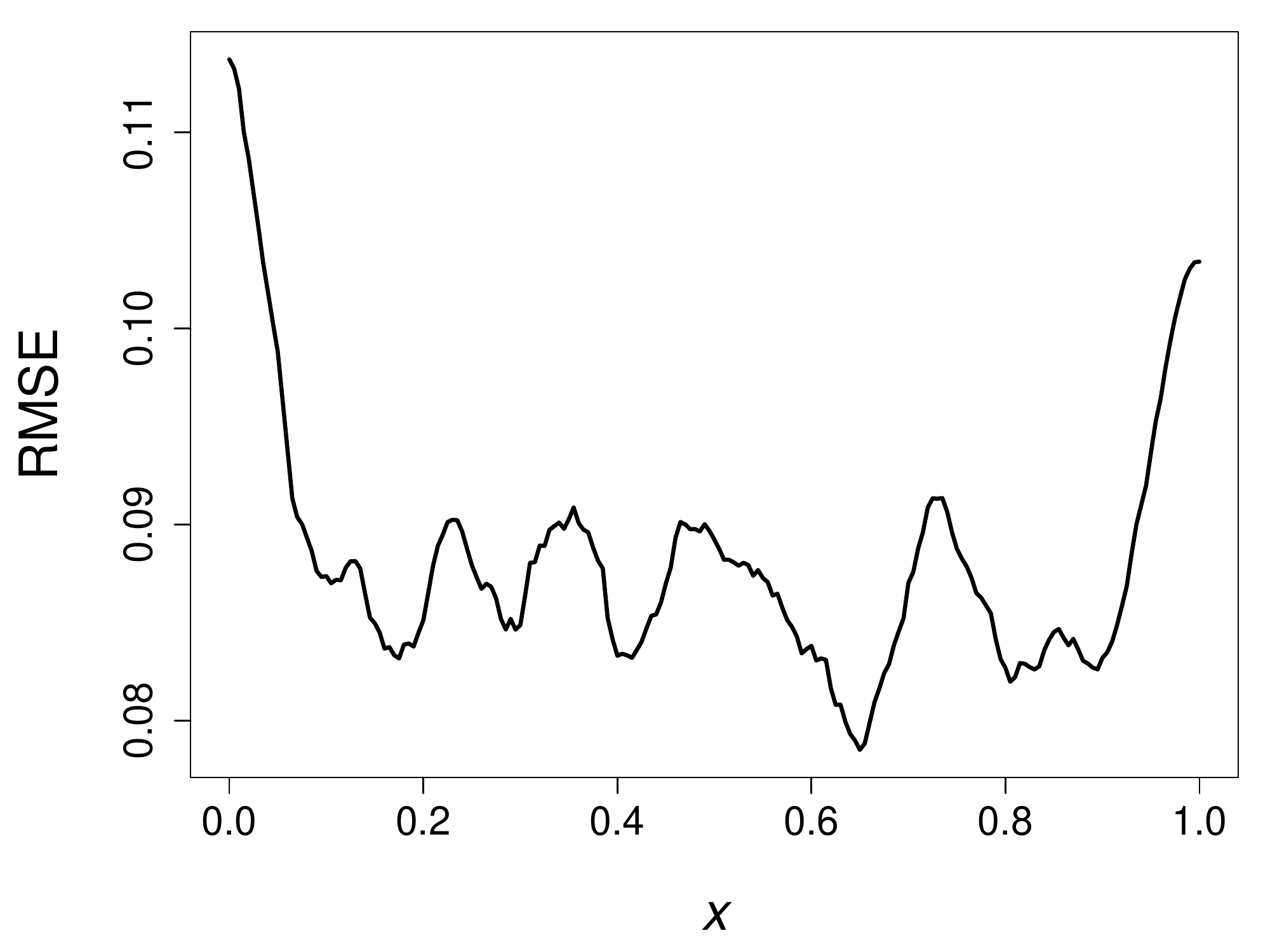}
    \caption{Pointwise RMSE of random forest with $ s = 100 $ and $ m = 1 $.}
          \label{fig:boundary_forest}
          \end{subfigure}
          \hspace{1.5cm}
\begin{subfigure}[t]{0.4\textwidth}
    \centering
    \includegraphics[width=1\textwidth]{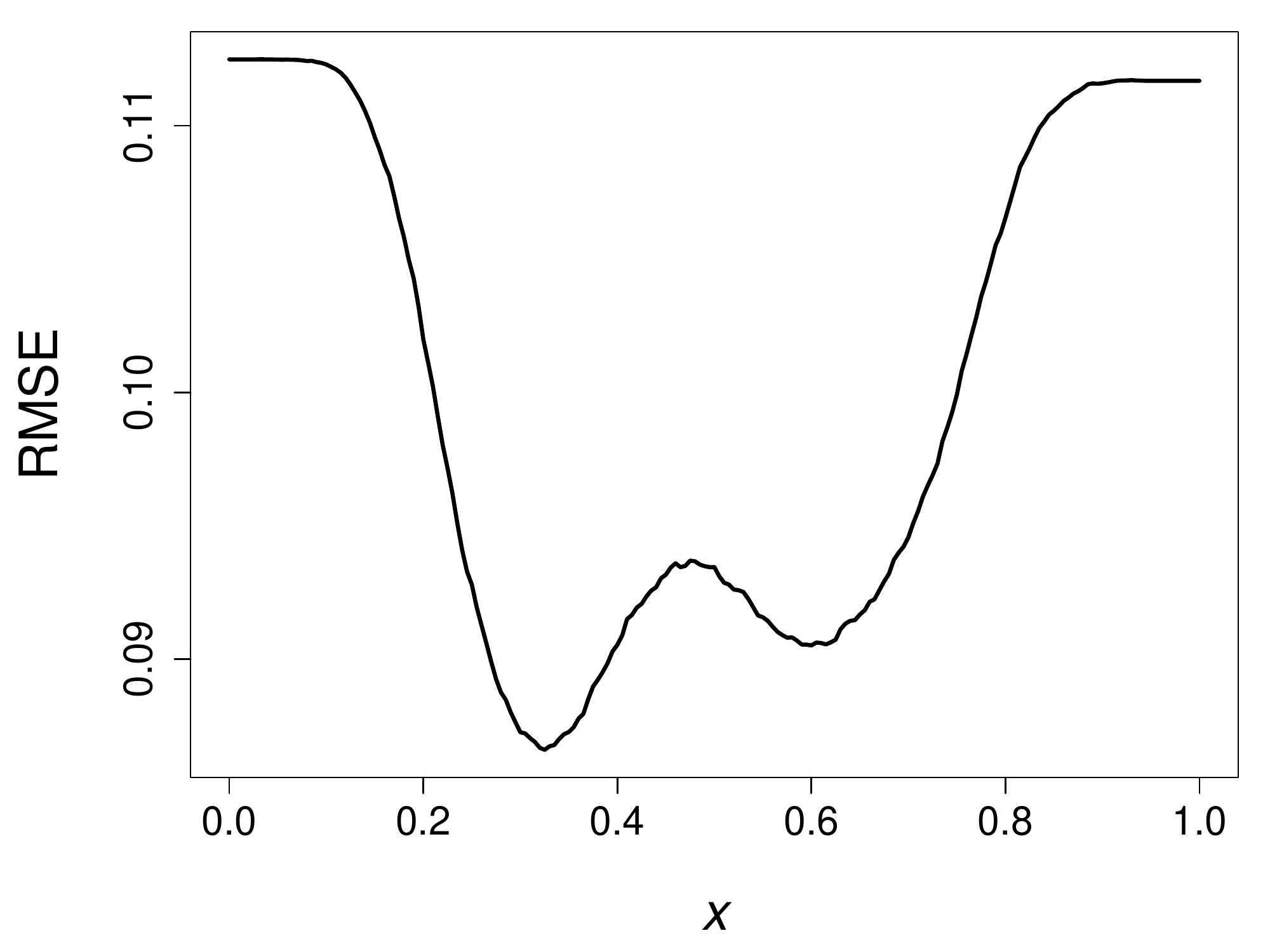}
    \caption{Pointwise RMSE of causal forest with $ s = 100 $ and $ m = 1 $.}
          \label{fig:boundary_forest_causal}
          \end{subfigure}
                    \caption{Pointwise RMSE of random forests for location model.}
                \label{fig:forest}
\end{figure}

In Figure \ref{fig:forest}, we plot the pointwise RMSE of a regression forest and causal forest for the respective models in Section \ref{sec:numerical_stump}, each time using $ B = 2000$ trees and the same sample sizes and number of replications as before. Compared to Figure \ref{fig:boundary} and Figure \ref{fig:boundary_causal}, we see that random forests have considerably better performance than a single tree near the boundary.

When $ p = 1 $, \cite{banerjee2007confidence} and \cite{buhlmann2002analyzing} investigated the properties of decision trees under assumptions that rule out the location model in Assumption \ref{ass:DGP}. They also showed that subsampling can reduce variance, similar to our result in Theorem \ref{thm:forests}. However, because the decision stump exhibits large bias in their setting, one cannot deduce from their results how random forests would improve the pointwise mean square error, which accounts for both bias and variance. Additionally, unlike Theorem \ref{thm:forests}, the random variable selection mechanism was not explored by \cite{banerjee2007confidence} and \cite{buhlmann2002analyzing} because their results are limited to the one-dimensional setting $ p = 1 $. As a consequence, Theorem \ref{thm:forests} complements prior literature by studying the pointwise mean squared error performance of random forest under the the location model with $p\geq 1$, and thus formalizes a beneficial aspect of random feature selection for decision tree ensembles.

Finally, while Theorem \ref{thm:forests} concerns a depth one ($K=1$) random forest construction, it is possible to explore multi-level honest tree ensembles. Theorem \ref{thm:main} showed that shallow honest trees constructed with the CART+ procedure can produce pointwise inconsistent estimates of the regression function $\mu$. In contrast, using the Hoeffding-Serfling variance inequality for U-statistics, it can be shown that an ensemble of depth $ K \asymp \log\log(n) $ trees constructed with CART+ methodology on subsampled data will have pointwise error $ \sqrt{\mathbb{E}[(\hat\mu(\bx) - \mu)^2]} = O(\sigma\sqrt{s/n})$, for all $\bx\in\mathcal{X}$. This result provides a concrete example where an ensemble of shallow inconsistent decision trees can be consistent with nearly optimal convergence rates, and is, to the best of our knowledge, the first time that such a result has been shown  in the literature for practical trees based on CART methodology.

\section{Conclusion}

This article studied the delicate pointwise properties of axis-aligned recursive partitioning, focusing on heterogeneous causal effect estimation, where accurate pointwise estimates over the entire support of the covariates are essential for valid statistical learning (e.g., point estimation, testing hypotheses, confidence interval construction). Specifically, we called into question the use of causal decision trees for such purposes by demonstrating that, for a standard location model, depth one decision trees (e.g., decision stumps) constructed using CART methodology exhibit pointwise convergence rates slower than any polynomial-in-$n$ in boundary regions of the support of the covariates. Even more dramatic, honest shallow decision trees were shown to be inconsistent even in large samples. Pruning was unable to overcome these limitations, but ensemble learning with both subsampling and random feature selection was successful at restoring near-optimal convergence rates for pointwise estimation for the specific simple class of data generating processes that we considered. While our emphasis was on direct use of decision trees for causal effect estimation, the methodological implications are similar for multi-step semi-parametric settings, where preliminary unknown functions (e.g., propensity scores) are estimated with machine learning tools, as well as conditional quantile regression, both of which require estimators with high pointwise accuracy.

In conclusion, our results have important implications for heterogeneous prediction and causal inference learning tasks employing decision trees. Whenever the goal is to produce accurate pointwise regression estimates over the entire support of the conditioning variables, even shallow decision trees trained with a large number of samples can exhibit poor performance. Consequently, adaptive recursive partitioning should be used with caution for heterogeneous prediction or causal inference purposes, especially in high-stakes environments where high pointwise accuracy is crucial.

\section*{Acknowledgements}

The authors thank Jianqing Fan, Max Farrell, Boris Hanin, Joowon Klusowski, Jantje S\"onksen, and Rocio Titiunik for comments.

\section*{Funding}

Cattaneo gratefully acknowledges financial support from the National Science Foundation through SES-1947805, SES-2019432, and SES-2241575. Klusowski gratefully acknowledges financial support from the National Science Foundation through CAREER DMS-2239448, DMS-2054808, and HDR TRIPODS CCF-1934924. Part of this research was conducted by Tian during his doctoral studies at Princeton University, and it constitutes a chapter in his dissertation.

\section*{Disclaimer}

This document is being distributed for informational and educational purposes only and is not an offer to sell or the solicitation of an offer to buy any securities or other instruments. The information contained herein is not intended to provide, and should not be relied upon for, investment advice.   The views expressed herein are not necessarily the views of Two Sigma Investments, LP or any of its affiliates (collectively, “Two Sigma”).  Such views reflect the assumptions of the author(s) of the document and are subject to change without notice. The document may employ data derived from third-party sources. No representation is made by Two Sigma as to the accuracy of such information and the use of such information in no way implies an endorsement of the source of such information or its validity.

The copyrights and/or trademarks in some of the images, logos or other material used herein may be owned by entities other than Two Sigma. If so, such copyrights and/or trademarks are most likely owned by the entity that created the material and are used purely for identification and comment as fair use under international copyright and/or trademark laws. Use of such image, copyright or trademark does not imply any association with such organization (or endorsement of such organization) by Two Sigma, nor vice versa.

\appendix

\section{Proofs}

In this appendix, we include proofs of the formal statements in the main text. Throughout the proofs below, because the quantities of interest are location invariant and homogeneous with respect to scale, by working with the standardized response variable $ (y_i-\mu)/\sigma $, we can assume without loss of generality that $\mu = 0$ and $ \sigma^2 = 1$. 
 
\subsection{Decision Stumps}

In this section, we prove \eqref{eq:split_range} in Theorem \ref{thm:master};
\eqref{eq:rate_constant2} and \eqref{eq:rate_constant} in Theorem \ref{thm:rates}; and \eqref{eq:honest_lower} and \eqref{eq:imse_rate}. Throughout this section, we denote the partial sum by $S_k = y_{1}+\cdots+y_{k}$, for $ k \geq 1$. 
\begin{proof}[Proof of \eqref{eq:split_range} in Theorem \ref{thm:master}]
Fix $ a, b \in (0, 1) $ with $ a < b $. According to \eqref{eq:Delta}, the desired probability is
\begin{equation}
\begin{aligned} \label{eq:master_prob}
& \mathbb{P}\big( n^a \leq \hat\imath \leq n^b\big)  \\ & \quad  = \mathbb{P}\Bigg( \max_{1 \leq k < n} \frac{\Big(\frac{1}{\sqrt{n}}S_k - \frac{k}{n}\frac{1}{\sqrt{n}}S_n \Big)^2}{(k/n)(1-k/n)} > \max_{1 \leq k < n^{a}, \; n^{b} < k < n}\frac{\Big(\frac{1}{\sqrt{n}}S_k - \frac{k}{n}\frac{1}{\sqrt{n}}S_n \Big)^2}{(k/n)(1-k/n)}  \Bigg).
\end{aligned}
\end{equation}
By \citet[Equation A.4.37]{csorgo1997limit}, we can define a sequence of Brownian bridges $ \{ B_n(t) : 0 \leq t \leq 1 \} $ on a suitable probability space such that
\begin{equation} \label{eq:error_full}
\Bigg|\max_{1 \leq k < n} \frac{\Big|\frac{1}{\sqrt{n}}S_k - \frac{k}{n}\frac{1}{\sqrt{n}}S_n \Big|}{\sqrt{(k/n)(1-k/n)}} - \sup_{1/n \leq t \leq 1-1/n}\frac{|B_n(t)|}{\sqrt{t(1-t)}} \Bigg| = \epsilon_n,
\end{equation}
where $ \epsilon_n = o_P\big((\log\log(n))^{-1/2}\big) $.
We note that while \citet[Equation A.4.37]{csorgo1997limit} bounds the approximation error of the maximum over the full range $ 1 \leq k < n $ as in \eqref{eq:error_full}, its proof, which relies on invariance principles for partial sums of i.i.d. random variables, can be generalized to bound the approximation error over $ 1 \leq k < n^{a} $, $ n^{b} < k < n $. Thus,
\begin{equation} \label{eq:error_partial}
\Bigg|\max_{1 \leq k < n^{a}, \; n^{b} < k < n} \frac{\Big|\frac{1}{\sqrt{n}}S_k - \frac{k}{n}\frac{1}{\sqrt{n}}S_n \Big|}{\sqrt{(k/n)(1-k/n)}}- \sup_{1/n \leq t < n^{a-1}, \; n^{b-1} < t \leq 1-1/n}\frac{|B_n(t)|}{\sqrt{t(1-t)}} \Bigg|  = \epsilon_n.
\end{equation}
Combining the approximations \eqref{eq:error_full} and \eqref{eq:error_partial}, the probability \eqref{eq:master_prob} can thus be lower bounded by
$$
\mathbb{P}\Bigg( \sup_{ 1/n \leq t \leq 1-1/n} \frac{|B_n(t)|}{\sqrt{t(1-t)}} > \sup_{1/n \leq t < n^{a-1}, \; n^{b-1} < t \leq 1-1/n} \frac{|B_n(t)|}{\sqrt{t(1-t)}} + 2\epsilon_n\Bigg).
$$
Next, we note that the standardized Brownian bridge $ \big\{ B_n(t)/\sqrt{t(1-t)} : 0 < t < 1 \big\} $ is distributionally equivalent to a time-transformed Ornstein-Uhlenbeck (O-U) process $ \big\{ U(\log(t/(1-t))) : 0 < t < 1\big\} $, where $ \big\{U(t): t \in \mathbb{R}\big\}$ is a zero-mean O-U process \citep[Section 1.9]{csorgo1981strong}, and thus
\begin{equation}
\begin{aligned} \label{eq:prob_orn}
& \mathbb{P}\Bigg( \sup_{ 1/n \leq t \leq 1-1/n} \frac{|B_n(t)|}{\sqrt{t(1-t)}} > \sup_{1/n \leq t < n^{a-1}, \; n^{b-1} < t \leq 1-1/n} \frac{|B_n(t)|}{\sqrt{t(1-t)}} + 2\epsilon_n\Bigg) \\ & = \mathbb{P}\Bigg( \sup_{ -\log(n-1) \leq t \leq \log(n-1)} |U(t)| > \sup_{-\log(n-1) \leq t < \log(n^{a-1}/(1-n^{a-1})), \; \log(n^{b-1}/(1-n^{b-1})) < t \leq \log(n-1)} |U(t)| + 2\epsilon_n\Bigg) \\
& =  \mathbb{P}\Bigg( \sup_{ 0 \leq t \leq 2\log(n-1)} |U(t)| > \sup_{ 0 \leq t < \log(n^{a-1}(n-1)/(1-n^{a-1})), \; \log(n^{b-1}(n-1)/(1-n^{b-1}))  < t \leq 2\log(n-1)} |U(t)| + 2\epsilon_n \Bigg) \\
& = \mathbb{P}\Bigg( \sup_{ 0 \leq t \leq 2\log(n-1)} |U(t)| > \sup_{ 0 \leq t \leq \log( (n-1)^2n^{a-b}(1-n^{b-1})/(1-n^{a-1})) } |U(t)| + 2\epsilon_n \Bigg),
\end{aligned}
\end{equation}
where the last two equalities result from, respectively, stationarity and the Markov property of the process $ |U(t)| $, the square of which is a Cox-Ingersoll-Ross (CIR) process \citep{jaeschke2003survey}.

By the Darling-Erd\H{o}s Limit Theorem for the O-U process \citep[Theorem A.3.1]{csorgo1997limit} and \citep[Theorem 2.2]{eicker1979asymptotic}, for all $ c > 0 $ and $ z \in \mathbb{R} $, we have
\begin{equation}
\begin{aligned}\label{eq:darling-erdos}
& \lim_{n\to\infty} \mathbb{P}\Bigg(\sup_{ 0 \leq t \leq (c+o(1))\log(n)}|U(t)| <  \frac{2\log\log(n) + (1/2)\log\log\log(n) + z - (1/2)\log(\pi)}{\sqrt{2\log\log(n)}}  \Bigg) \\ & \qquad\qquad\qquad =
 e^{-e^{-(z-2\log(c))}}.
 \end{aligned}
\end{equation}
Let $ z^* $ maximize $ z \mapsto e^{-e^{-(z-2\log(2-(b-a)))}} - e^{-e^{-(z-2\log(2))}} $. Simple calculus yields $$ z^* = \log\Bigg(\frac{v(4-v)}{\log(4/(2-v)^2)}\Bigg), \qquad v = b-a. $$ Define $$ u_n = \frac{2\log\log(n) + (1/2)\log\log\log(n) + z^* - (1/2)\log(\pi)}{\sqrt{2\log\log(n)}}. $$
Continuing from \eqref{eq:prob_orn}, using \eqref{eq:darling-erdos} twice with $ c = 2 $ and $ c = 2-(b-a) $ and the fact that $ \epsilon_n = o_P\big((\log\log(n))^{-1/2}\big) $, we obtain
\begin{equation} 
\begin{aligned} \label{eq:prob_lower}
& \liminf_{n\to\infty} \mathbb{P}\Bigg( \sup_{ 0 \leq t \leq 2\log(n-1)} |U(t)| > \sup_{ 0 \leq t \leq \log( (n-1)^2n^{a-b}(1-n^{b-1})/(1-n^{a-1})) } |U(t)| + 2\epsilon_n \Bigg) \\ & \qquad \geq \liminf_{n\to\infty} \mathbb{P}\Bigg( \sup_{ 0 \leq t \leq 2\log(n-1)} |U(t)| \geq u_n, \quad \sup_{ 0 \leq t \leq \log( (n-1)^2n^{a-b}(1-n^{b-1})/(1-n^{a-1})) } |U(t)| < u_n - 2\epsilon_n \Bigg) \\ &
 \qquad \geq
  \lim_{n\to\infty}  \mathbb{P}\Bigg( \sup_{ 0 \leq t \leq \log( (n-1)^2n^{a-b}(1-n^{b-1})/(1-n^{a-1})) } |U(t)| < u_n - 2\epsilon_n \Bigg) - \lim_{n\to\infty}\mathbb{P}\Bigg( \sup_{ 0 \leq t \leq 2\log(n-1)} |U(t)| < u_n \Bigg)
  \\ & \qquad
 = \lim_{n\to\infty} \Big( e^{-e^{-(z^*+o(1)-2\log(2-(b-a)))}} + o(1) \Big) - e^{-e^{-(z^*-2\log(2))}} \\ & \qquad =
 e^{-e^{-(z^*-2\log(2-(b-a)))}} - e^{-e^{-(z^*-2\log(2))}} \\ & \qquad =
v \frac{(4-v)(1-v/2)^{8/(v(4-v))}}{(2-v)^2}
\\ & \qquad \geq v \cdot \lim_{u \downarrow 0}\frac{(4-u)(1-u/2)^{8/(u(4-u))}}{(2-u)^2}
\\ & \qquad = (b-a)/e.
\end{aligned}
\end{equation}
We have thus shown that $ \liminf_{n\to\infty} \mathbb{P}( n^a \leq \hat\imath \leq n^b) \geq (b-a)/e $. By symmetry, we obtain $ \liminf_{n\to\infty} \mathbb{P}(n-n^b \leq \hat\imath \leq n- n^a) \geq (b-a)/e $, and by disjointness of the events $ n^a \leq \hat\imath \leq n^b $ and $ n-n^b \leq \hat\imath \leq n- n^a $ when $ n > 2^{1/(1-b)} $, we also have
\begin{equation*}
\liminf_{n\to\infty} \mathbb{P}( n^a \leq \hat\imath \leq n^b \;\; \text{or} \;\; n-n^b \leq \hat\imath \leq n- n^a) \geq (2/e)(b-a). \qedhere
\end{equation*}

\begin{remark}
Alternatively, for any $ 0 \leq A < B $, we have
\begin{equation} \label{eq:max-ou}
\mathbb{P}\Bigg( \sup_{0 \leq t \leq B} |U(t)| > \sup_{0 \leq t \leq A} |U(t)| \Bigg) = \frac{B-A}{B}.
\end{equation}
This can readily be shown using the fact that the absolute value of a zero-mean O-U process is stationary, Markov, and has continuous paths. Consequently, ignoring the stochastic error $\epsilon_n$ from approximating the impurity gain \eqref{eq:Delta} by the square of a standardized Brownian bridge (not yet justified), using \eqref{eq:max-ou}, we can approximate the probability \eqref{eq:master_prob} by
\begin{align}
& \mathbb{P}\Bigg( \sup_{ 0 \leq t \leq 2\log(n-1)} |U(t)| > \sup_{ 0 \leq t \leq \log( (n-1)^2n^{a-b}(1-n^{b-1})/(1-n^{a-1})) } |U(t)| \Bigg) 
\\ & \qquad = \frac{2\log(n-1)-\log( (n-1)^2n^{a-b}(1-n^{b-1})/(1-n^{a-1}))}{2\log(n-1)} \rightarrow \frac{b-a}{2}, \quad n \rightarrow \infty.
\end{align}
We are therefore led to conjecture that
$$
\lim_{n\to\infty} \mathbb{P}( n^a \leq \hat\imath \leq n^b) = \frac{b-a}{2}.
$$
\end{remark}
\end{proof}

\begin{proof}[Proof of \eqref{eq:rate_constant2} and \eqref{eq:rate_constant} in Theorem \ref{thm:rates}]

By \citet[Theorem A.4.1]{csorgo1997limit} and Donsker's Theorem which says that $ \max_{n/2 < k < n}|S_k|/\sqrt{k} $ and $ \max_{1 \leq k \leq n/2}|S_n-S_k|/\sqrt{n-k} $ converge in distribution to $ \sup_{0\leq t \leq \log(2)}|U(t)| $, we have
\begin{equation} \label{eq:LIL}
\frac{\max_{1 \leq k < n}\frac{|S_k|}{\sqrt{k}}}{\sqrt{2\log\log(n)}} = 1 + o_P(1), \qquad \frac{\max_{n/2 < k < n}\frac{|S_k|}{\sqrt{k}}+\max_{1 \leq k \leq n/2}\frac{|S_n-S_k|}{\sqrt{n-k}}}{\sqrt{2\log\log(n)}} = o_P(1).
\end{equation}
To prove \eqref{eq:rate_constant2}, we note that, on the event $ \hat\imath \leq n^b $ or $ \hat\imath \geq n-n^b $ which occurs with asymptotic probability at least $ 2b/e $, we have
\begin{align}
\sup_{x \in \mathcal{X}} |\hat \mu(T_1)(x)|^2 
&\geq 
\frac{S_{\hat\imath}^2}{\hat\imath^2}\Indicator(\hat\imath \leq n/2) +
\frac{(S_n - S_{\hat\imath})^2}{(n-\hat\imath)^2}\Indicator(\hat\imath > n/2)\\
&\geq 
\frac{1}{\min\{\hat\imath, \, n - \hat\imath\}}\Bigg(
\frac{S_{\hat\imath}^2}{\hat\imath} +
\frac{(S_n - S_{\hat\imath})^2}{n-\hat\imath}
- \Bigg(\frac{S_{\hat\imath}^2}{\hat\imath} \Indicator(\hat\imath > n/2) +
\frac{(S_n - S_{\hat\imath})^2}{n-\hat\imath} \Indicator(\hat\imath \leq n/2)\Bigg)\Bigg)\\
& \geq \frac{1}{\min\{\hat\imath, \, n - \hat\imath\}}\Bigg(\max_{1 \leq k < n}\Bigg(\frac{S_k^2}{k}+\frac{(S_n-S_k)^2}{n-k}\Bigg) - \max_{n/2 < k < n} \frac{S^2_k}{k} - \max_{1 \leq k \leq n/2}\frac{(S_n-S_k)^2}{n-k}\Bigg) \\
& \geq  \frac{(2+o_P(1))\log\log(n)}{\min\{\hat\imath, \, n - \hat\imath\}} \\
&\geq \frac{(2+o_P(1))\log\log(n)}{n^b}.
\end{align}
Here we have used the fact that $ \hat\imath $ equivalently maximizes $ k \mapsto S^2_k/k + (S_n-S_k)^2/(n-k) = \sum_{i=1}^n(y_i-\mu)^2 - \sum_{i = 1}^n(y_i - (S_k/k) \Indicator(i \leq k) - ((S_n-S_k)/(n-k))\Indicator(i > k))^2 $ over $ 1 \leq k < n $.

The other result \eqref{eq:rate_constant} follows again from \eqref{eq:LIL} and from the fact that $ x_{(n^a)}/n^{a-1} = 1 + o_P(1) $ and $ x_{(n-n^a)}/(1-n^{a-1}) = 1+o_P(1) $. Thus, on the event $ n^a \leq \hat\imath \leq n^b $ which occurs with asymptotic probability at least $(b-a)/e$, if $ x_n \leq (1+o_P(1))n^{a-1} = x_{(n^a)} \leq x_{(\hat\imath)} = \hat\tau $, we have
\begin{align}
|\hat \mu(T_1)(x_n)|^2 = \frac{S_{\hat\imath}^2}{\hat\imath^2}
& =
\frac{1}{\hat\imath}\Bigg(\frac{S_{\hat\imath}^2}{\hat\imath} +
\frac{(S_n - S_{\hat\imath})^2}{n-\hat\imath} -  \frac{(S_n - S_{\hat\imath})^2}{n-\hat\imath}\Bigg)\\
& \geq \frac{1}{\hat\imath}\Bigg(\max_{1 \leq k < n}\Bigg(\frac{S_k^2}{k}+\frac{(S_n-S_k)^2}{n-k}\Bigg) - \max_{1 \leq k \leq n^b}\frac{(S_n-S_k)^2}{n-k}\Bigg) \\
& =  \frac{(2+o_P(1))\log\log(n)}{\hat\imath}\\
&\geq \frac{(2+o_P(1))\log\log(n)}{n^b}.
\end{align}
By symmetry, on the event $ n-n^b \leq \hat\imath \leq n-n^a $, the same lower bound holds for $x_n > 1 - (1+o_P(1))n^{a-1}$.
\end{proof}

\begin{proof}[Proof of \eqref{eq:honest_lower}]
We first observe that
\begin{align}
\mathbb{E}\Big[\sup_{x\in\mathcal{X}}(\tilde\mu(T_1)(x))^2\Big] & \geq \text{Var}(\tilde \mu(T_1)(0)) \\ & = 
\mathbb{E}\Bigg[\Bigg(\frac{\Indicator(\#\{\tilde x_i \leq x_{(\hat\imath)}\} > 0)}{\#\{\tilde x_i \leq x_{(\hat\imath)}\}}\sum_{i=1}^n \tilde y_i \Indicator(\tilde x_i \leq x_{(\hat\imath)})\Bigg)^2 \Bigg] \\
& = \mathbb{E}\Bigg[\frac{\Indicator(\#\{\tilde x_i \leq x_{(\hat\imath)}\} > 0)}{\#\{\tilde x_i \leq x_{(\hat\imath)}\}}\Bigg],
\end{align}
where we used the independence between $\tilde y_i $ and $ \tilde x_i  $ and $ x_i $, per the honest construction and Assumption \ref{ass:DGP}.
By the Cauchy-Schwarz inequality, we have
\begin{equation} \label{eq:cauchy}
\mathbb{E}\Bigg[\frac{\Indicator(\#\{\tilde x_i \leq x_{(\hat\imath)}\} > 0)}{\#\{\tilde x_i \leq x_{(\hat\imath)}\}}\Bigg] \geq \mathbb{E}\Bigg[\frac{(\mathbb{P}(\#\{\tilde x_i \leq x_{(\hat\imath)}\} > 0 \mid \hat\imath))^2}{\mathbb{E}[\#\{\tilde x_i \leq x_{(\hat\imath)}\}\mid \hat\imath]}\Bigg].
\end{equation}
Again, by the honest construction and Assumption \ref{ass:DGP}, we note that $ \tilde x_i  $, $ x_i $, and $\hat\imath $ are mutually independent. In particular, $ x_{(\hat\imath)} $ given $ \hat\imath = i $ is distributed $ \text{Beta}(i,\, n-i+1) $, allowing us to compute
\begin{align}
\mathbb{P}\big(\#\{\tilde x_i \leq x_{(\hat\imath)}\} > 0 \mid \hat\imath \big) & = 1 - \mathbb{E}[(1-x_{(\hat\imath)})^n \mid \hat\imath] \\ & = 1 - \binom{2n-\hat\imath}{n}/\binom{2n}{n} \\ & = 1 -  \prod_{i=1}^{\hat\imath}\frac{n-i+1}{2n-i+1} \\ & \geq 1- 2^{-\hat\imath},
\end{align}
and
$$
\mathbb{E}[\#\{\tilde x_i \leq x_{(\hat\imath)}\}\mid \hat\imath] = \frac{n}{n+1}\hat\imath.
$$
We may thus lower bound \eqref{eq:cauchy} via
$$
\mathbb{E}\Bigg[\frac{\Big(1-\binom{2n-\hat\imath}{n}/\binom{2n}{n}\Big)^2}{n\hat\imath/(n+1)}\Bigg] \geq \mathbb{E}\Bigg[\frac{(1-2^{-\hat\imath})^2}{\hat\imath}\Bigg].
$$
The fact that $ \mathbb{E}\Big[\frac{(1-2^{-\hat\imath})^2}{\hat\imath}\Big] \gtrsim n^{-b} $ follows directly from \eqref{eq:split_range}.
\end{proof}

\subsection{Inconsistency with Deeper Trees}
In this section, we prove Theorem \ref{thm:main}. First, we define some notation related to the tree construction which will be used in the proofs. Let $\tilde n_k$ be the number of observations in the left-most cell (i.e., the node containing $x = 0$) at depth $k$ and $\tilde\imath_{k} $ be the CART split index of this node, with $\tilde n_0 = n$ and $\tilde\imath_0 = \hat\imath$ (recall that $ \hat\imath $ is the split index for the decision stump \eqref{eq:stump}). Then, the left-most cell at the $k$-th level can be expressed as $ [0, x_{(\tilde\imath_{k-1})}] $ and $ \tilde n_k = \tilde \imath_{k-1} = \#\{ x_i \leq x_{(\tilde\imath_{k-1})}\} $.

\begin{lemma} \label{lmm:recursion}
There exist $ \delta \in (0, 1) $, $c > 1$, and a positive integer $M$ such that for any depth $k\geq 1$ and $m\geq M$, we have
$
\mathbb{P}(\tilde n_k \leq m) \geq  (1-\delta)\cdot \mathbb{P}(\tilde n_{k-1}\leq m) + \delta \cdot \mathbb{P}(\tilde n_{k-1}\leq m^c).
$
\end{lemma}
\begin{proof}
Observe that if $m$ is a positive integer, then $\tilde\imath_{k-1} \mid \tilde n_{k-1} = m$ has the same distribution as $\tilde\imath_0 \mid \tilde n_0 = m$, because of the honest tree construction and Assumption \ref{ass:DGP}.
Therefore, we can apply \eqref{eq:split_range} to obtain
\begin{equation} \label{eq:conditional_split}
\mathbb{P}\big( m^a \leq \tilde\imath_{k-1} \leq 
m^b \mid \tilde n_{k-1} = m\big) \geq \delta > 0,
\end{equation}
for some $ \delta > 0 $ and sufficiently large $ m $. Hence, by \eqref{eq:conditional_split}, we have for $ m $ sufficiently large,
\begin{equation} \label{eq:cond}
\begin{aligned}
& \mathbb{P}\big(\tilde n_k  \leq m \mid m < \tilde n_{k-1} \leq m^{1/b} \big) \\ &
\quad \geq \min_{m < i \leq m^{1/b} } \mathbb{P}\big(i^a \leq \tilde\imath_{k-1}\leq i^b \mid \tilde n_{k-1} = i\big)\mathbb{P}\big(\tilde n_k  \leq m \mid i^a \leq \tilde\imath_{k-1} \leq i^b \big) \\
& \quad \geq \delta \min_{m < i \leq m^{1/b} }  \mathbb{P}\big(\tilde n_k  \leq m \mid i^a \leq \tilde\imath_{k-1} \leq i^b \big) \\
& \quad \geq \delta\min_{m^a < i \leq m } \mathbb{P} \big(\tilde n_k  \leq \tilde\imath_{k-1} \mid \tilde\imath_{k-1}  = i \big) \\
& \quad = \delta.
\end{aligned}
\end{equation}
Now, taking $ c = 1/b $, note that
\eqref{eq:cond} implies Lemma \ref{lmm:recursion} since, for $ m $ sufficiently large, we have
\begin{align}
\mathbb{P}(\tilde n_k \leq m)
& = \mathbb{P}(\tilde n_k \leq m, \; \tilde n_{k-1} > m^c) + \mathbb{P}(\tilde n_k \leq m, \; \tilde n_{k-1} \leq m^c) \\
& \geq \mathbb{P}(\tilde n_k \leq m, \; \tilde n_{k-1} \leq m^c) \\
& = \mathbb{P}(\tilde n_k \leq m, \; \tilde n_{k-1} \leq m) + \mathbb{P}(\tilde n_k \leq m, \; m < \tilde n_{k-1} \leq m^c) \\
& \geq \mathbb{P}(\tilde n_{k-1}\leq m)+ \delta\cdot\mathbb{P}(m < \tilde n_{k-1}\leq m^c) \\
& = (1-\delta)\cdot \mathbb{P}(\tilde n_{k-1}\leq m)+ \delta\cdot\mathbb{P}(\tilde n_{k-1}\leq m^c). \qedhere
\end{align}
\end{proof}
Next, we use Lemma \ref{lmm:recursion} to finish the proof of Theorem \ref{thm:main}. The main idea is to establish that the terminal nodes in a shallow tree will be small with constant probability.
\begin{proof}[Proof of Theorem \ref{thm:main}]
Define $n_\ell = n^{(1/c)^{\ell}}$.
We will show by induction that for any $k \geq 0$ and $\ell \geq 1$ such that $n_\ell \geq M$,
\begin{equation} \label{eq:induction}
\mathbb{P}(\tilde n_k \leq n_{\ell}) \geq
\sum_{k'=\ell}^{k} {k'-1 \choose \ell-1} (1-\delta)^{k'-\ell} \delta^\ell.
\end{equation}
The base case of $k = 0$ is trivial since $\tilde n_0 = n$.
Now, assume that for some fixed $k \geq 1$ and any $\ell' \geq 1$ such that $n_{\ell'} \geq M$, we have
\begin{equation} \label{eq:hypothesis}
\mathbb{P}(\tilde n_{k-1}\leq n_{\ell'}) \geq
\sum_{k'=\ell'}^{k-1} {k'-1 \choose \ell'-1} (1-\delta)^{k'-\ell'} \delta^{\ell'}.
\end{equation}
If $\ell \geq 2$, then substituting our induction hypothesis \eqref{eq:hypothesis} with $\ell' = \ell$ and $\ell'= \ell-1$ into Lemma \ref{lmm:recursion}, we get that
\begin{align}
\mathbb{P}(\tilde n_k \leq n_\ell)
& \geq (1-\delta) \sum_{k'=\ell}^{k-1} {k'-1 \choose \ell-1} (1-\delta)^{k'-\ell} \delta^\ell + \delta \sum_{k'=\ell-1}^{k-1} {k'-1 \choose \ell-2} (1-\delta)^{k'-\ell+1}\delta^{\ell-1} \\
& =\sum_{k'=\ell}^k {k'-1 \choose \ell-1}(1-\delta)^{k'-\ell} \delta^\ell,
\end{align}
where we used Pascal's identity.
This completes the inductive proof of \eqref{eq:induction}.

Let $X \sim \text{NB}(L, \delta)$, i.e., the number of independent trials, each occurring with probability $\delta$, until $L$ successes. Choose
 $$
L = \lceil \log_c\log_c(n) - \log_c\log_c(M)-1\rceil \asymp \log\log(n), \quad n_L =   n^{(1/c)^L} \in [M,M^c].
 $$
By \eqref{eq:induction} and Markov's inequality applied to the tail probability of $ X $, we have that
\begin{equation}
 \begin{aligned} \label{eq:terminal}
 \mathbb{P}(\tilde n_K \leq n_L)  & \geq
 \sum_{k'=L}^K {k'-1 \choose L-1} (1-\delta)^{k'-L}\delta^L \\ & = 1 - \mathbb{P}(X \geq K + 1) \\ & \geq 1- \frac{\mathbb{E}[X]}{K+1} \\ & =  1 - \frac{L}{\delta(K+1)} \\ & \geq \frac{1}{2},
 \end{aligned} 
 \end{equation}
as long as $ K \geq 2L/\delta \gtrsim \log\log(n)$. By the Paley-Zygmund inequality \citep{petrov2007lower} and the fact that $ \text{Var}(\tilde \mu(T_{K})(0)) = \mathbb{E}[1/\tilde n_K] \leq 1 $,
we have
\begin{equation} \label{eq:paley}
\mathbb{P}\Bigg(|\tilde \mu(T_{K})(0)| > \frac{\mathbb{E}[|\tilde \mu(T_{K})(0)|]}{2} \Bigg) \geq \frac{(\mathbb{E}[|\tilde \mu(T_{K})(0)|])^2}{4\text{Var}(\tilde \mu(T_{K})(0))} \geq \frac{(\mathbb{E}[|\tilde \mu(T_{K})(0)|])^2}{4} . 
\end{equation}
By the honest construction of the tree and \eqref{eq:terminal}, we have the lower bound
\begin{equation}
\begin{aligned} \label{eq:expect_paley}
\mathbb{E}[|\tilde \mu(T_{K})(0)|] & = \sum_{k=1}^n \mathbb{E}\Bigg[\Bigg|\frac{1}{k}\sum_{i=1}^k \tilde y_i \Bigg|\Bigg]\mathbb{P}(  \tilde n_K = k) \\
&  \geq \min_{ k = 1, 2, \dots, \lceil n_L \rceil }\mathbb{E}\Bigg[\Bigg|\frac{1}{k}\sum_{i=1}^k \tilde y_i \Bigg|\Bigg]\mathbb{P}(  \tilde n_K \leq \lceil n_L \rceil) \\ & \geq \frac{1}{2}\min_{ k = 1, 2, \dots, \lceil n_L \rceil}\mathbb{E}\Bigg[\Bigg|\frac{1}{k}\sum_{i=1}^k \tilde y_i \Bigg|\Bigg].
\end{aligned}
\end{equation}
Since a sum of independent random variables is almost surely constant if and only if the individual random variables are almost surely constant, it follows that the last expression in \eqref{eq:expect_paley} is bounded away from zero. Returning to \eqref{eq:paley} completes the proof.
\end{proof}

\begin{proof}[Proof of \eqref{eq:imse_rate}]
Let $0 = \tilde \imath_0 < \tilde \imath_1 \leq \cdots \leq \tilde \imath_{2^K-1} < \tilde\imath_{2^K} = n $ and $ 0 = \tilde\tau_0 < \tilde \tau_1 \leq \cdots \leq \tilde \tau_{2^K-1} < \tilde\tau_{2^K} = 1 $ denote the successive splits indices and values, respectively, at the terminal level of the tree (if a node cannot be further refined, we duplicate the split indices and values at the next level). Note that the split indices are independent of the $\tilde y_i $ data by the honest condition and the $x_i $ data per Assumption \ref{ass:DGP}. In particular, note that $ \tilde\tau_k = x_{(\tilde\imath_k)} $ given $ \tilde\imath_k = i $ is distributed $ \text{Beta}(i,\, n-i+1) $.  Thus, the IMSE can be bounded as follows:
\begin{align}
\mathbb{E}\Bigg[\int_{\mathcal{X}}(\tilde \mu(T_K)(x))^2 \mathbb{P}_x(dx)\Bigg] & =
\sum_{k=1}^{2^K}\mathbb{E}\Bigg[(\tilde\tau_k - \tilde\tau_{k-1})\Bigg(\frac{\Indicator(\tilde\imath_k > \tilde\imath_{k-1})}{\tilde\imath_k-\tilde\imath_{k-1}}\sum_{i=1}^n \tilde y_i\Indicator(\tilde\tau_{k-1} \leq x_i < \tilde\tau_k)\Bigg)^2\Bigg]
 \\ & =
\mathbb{E}\Bigg[\frac{\tilde\tau_{1}}{\tilde\imath_{1}}\Bigg] +  \sum_{k=2}^{2^K-1}\mathbb{E}\Bigg[\frac{\tilde\tau_k-\tilde\tau_{k-1}}{\tilde\imath_k - \tilde\imath_{k-1}}\Indicator(\tilde\imath_k > \tilde\imath_{k-1})\Bigg] + \mathbb{E}\Bigg[\frac{1-\tilde\tau_{2^K-1}}{n- \tilde\imath_{2^K-1}}\Bigg]\\
& \leq  \mathbb{E}\Bigg[\frac{1}{n+1}\Bigg] +  \sum_{k=2}^{2^K-1}\mathbb{E}\Bigg[\frac{1}{n+1}\Bigg] + \mathbb{E}\Bigg[\frac{1}{n+1}\frac{n- \tilde\imath_{2^K-1}+1}{n- \tilde\imath_{2^K-1}}\Bigg] \\
& \leq \frac{2^{K+1}}{n+1}.
\qedhere
\end{align}
\end{proof}

\subsection{Random Forests}
In this section, we prove Theorem \ref{thm:forests}. The following lemmas will be helpful.
\begin{lemma} \label{lmm:bin}
If $W \sim \text{Bin}(w, r)$, where $w\in \mathbb{N} $ and $ r \in (0, 1] $, then
$\mathbb{E}\big[\frac{1}{W+1}\big] \leq \frac{1}{(w+1)r}$.
\end{lemma} 
\begin{proof}
We have
\begin{align}
\mathbb{E}\Bigg[\frac{1}{W+1}\Bigg] & = \sum_{i=0}^w \frac{1}{i+1} {w \choose i}r^i(1-r)^{w-i}
= \frac{1}{(w+1)r}\sum_{i=1}^{w+1} {w+1 \choose i}r^{i}(1-r)^{w+1-i} 
\leq \frac{1}{(w+1)r}.
\qedhere
\end{align}
\end{proof}

\begin{lemma} \label{lmm:cap}
Let $ m $ and $ a $ be positive integers and $A$ and $A'$ be two independent random subsets of $\{1, 2,\dots,m\}$ of size $a$. Then,
$
\frac{1}{{m \choose a}^2} \sum_{A,A'} |A \cap A'| = \frac{a^2}{m}.
$
\end{lemma} 
\begin{proof}
We have
\begin{align}
\mathbb{E}_{A,A'}[|A \cap A'|]
&= \sum_{i \in \{1,2,\dots,m\}}\mathbb{E}[\Indicator(i \in A \cap A') ]
= \sum_{i \in \{1,2,\dots,m\}}\mathbb{P}(i \in A) \mathbb{P}(i \in A')
= m \cdot \frac{a}{m} \cdot \frac{a}{m} = \frac{a^2}{m}.
\qedhere
\end{align}
\end{proof}

\begin{lemma} \label{lmm:diff_product}
Let $(S_0,S_1)$ and $(S_0',S_1')$ be two independent subsamples
from the honest forest construction.
Then, we have
$$
\frac{1}{{n \choose s/2}^2{n-s/2 \choose s/2}^2}
\sum_{S_0,S_1}\sum_{S_0',S_1'}|S_1'\cap S_0||S_1\cap S_0'|
 \leq \frac{s^4}{16n(n-s/2)}.
$$
\end{lemma} 
\begin{proof}
First, assume that $S_1'$ and $S_0$ are fixed. Notice that $S_1 \cap S_0'$ is disjoint from $S_1' \cup S_0$.
Thus, we have 
\begin{equation} \label{eq:cap1}
\begin{aligned}
& \mathbb{E}[|S_1 \cap S_0'| \mid S_1', S_0] \\
&= \sum_{i \not \in S_1' \cup S_0}\mathbb{P}(i \in S_1 \cap S_0' \mid S_1', S_0
)= \sum_{i \not \in S_1' \cup S_0}\mathbb{P}(i \in S_1 \mid S_1',S_0) \mathbb{P}(i \in S_0' \mid S_1',S_0),\\
&= (n-|S_1' \cup S_0|) \Bigg(\frac{s/2}{n-s/2}\Bigg)^2
\leq \frac{s^2}{4(n-s/2)}.
\end{aligned}
\end{equation}
Combining \eqref{eq:cap1} and Lemma \ref{lmm:cap}, we have
\begin{align}
\frac{1}{{n \choose s/2}^2{n-s/2 \choose s/2}^2}
\sum_{S_0,S_1}\sum_{S_0',S_1'}|S_1'\cap S_0||S_1\cap S_0'|
& = \mathbb{E}[|S_1' \cap S_0| \cdot \mathbb{E}[|S_1 \cap S_0'| \mid S_1', S_0]] \\ & \leq  \frac{s^4}{16n(n-s/2)}.
\qedhere
\end{align}
\end{proof}

\begin{lemma}\label{lmm:diff}
Let $(S_0,S_1)$ and $(S_0',S_1')$ be two independent subsamples
from the honest forest construction. Given a fixed $S_1$ and $S_1'$ such that $|S_1 \cap S_1'| \geq 1$, we have
\begin{equation} \label{eq:diff2}
\frac{1}{{n-s/2 \choose s/2}}
\sum_{S_0}\frac{1}{|S_1'\backslash S_0|} - \frac{2}{s}=
\frac{1}{{n-s/2 \choose s/2}}
\sum_{S_0'}\frac{1}{|S_1\backslash S_0'|} - \frac{2}{s} \leq \frac{2n}{s(n-s+2)}.
\end{equation}
Furthermore,
\begin{equation} \label{eq:cond_indep}
\frac{1}{{n-s/2 \choose s/2}^2}
\Bigg(\sum_{S_0}\frac{1}{|S_1'\backslash S_0|} - \frac{2}{s} \Bigg)\Bigg(
\sum_{S_0'}\frac{1}{|S_1\backslash S_0'|} - \frac{2}{s}\Bigg) \leq \frac{4n^2}{s^2(n-s+2)^2}.
\end{equation}
\end{lemma}
\begin{proof}
Fix $S_1$ and $S_1'$ and note that $\mathbb{P}(|S_1 \cap S_0'| = k \mid S_1, S_1') = \frac{{s/2 - |S_1 \cap S_1'| \choose k}{n-s+|S_1 \cap S_1'|\choose s/2 - k} }{{n-s/2 \choose s/2}}$. Then,
\begin{equation} \label{eq:conditional}
\begin{aligned}
\frac{1}{{n-s/2 \choose s/2}}
\sum_{S_0'}\frac{1}{|S_1\backslash S_0'| }
&= \sum_{k=0}^{s/2-|S_1 \cap S_1'|} \frac{1}{s/2-k} \mathbb{P}(|S_1 \cap S_0'| = k \mid S_1, S_1') \\
&\leq \sum_{k=0}^{s/2-|S_1 \cap S_1'|} \frac{2}{s/2-k+1} \frac{{s/2 - |S_1 \cap S_1'| \choose k}{n-s+|S_1 \cap S_1'|\choose s/2 - k} }{{n-s/2 \choose s/2}} \\
&\leq \frac{2(n-s/2+1)}{(n-s+|S_1 \cap S_1'|+1)(s/2+1)}\sum_{k=0}^{s/2-|S_1 \cap S_1'|} \frac{{s/2 - |S_1 \cap S_1'| \choose k}{n-s+|S_1 \cap S_1'|+1 \choose s/2 - k + 1} }{{n-s/2 + 1 \choose s/2+1}} \\
&\leq \frac{4(n-s/2+1)}{s(n-s+2)},
\end{aligned}
\end{equation}
which implies that \eqref{eq:diff2} holds regardless of $(S_1, S_1')$. This implies \eqref{eq:cond_indep}, since $S_1\backslash S_0'$ is conditionally independent of $S_1' \backslash S_0$ given $(S_1, S_1')$.
\end{proof}

\begin{proof}[Proof of Theorem \ref{thm:forests}]

We use the notation $(\hat s(M,S_0),\hat \jmath(M,S_0))$ to denote the split point and direction, respectively, for a given pair $(M,S_0)$.
First, notice that 
\begin{equation} \label{eq:mu_mse}
\begin{aligned}
\mathbb{E}\big[(\hat \mu(\bx))^2\big]
& = \frac{1}{\binom{p}{m}^2 \binom{n}{s/2}^2 \binom{n-s/2}{s/2}^2}\sum_{M, M'} \sum_{S, S'}\mathbb{E}[\hat\mu(T(M, S))(\bx)\hat\mu(T(M',S'))(\bx)] \\
&= \frac{1}{\binom{p}{m}^2 \binom{n}{s/2}^2 \binom{n-s/2}{s/2}^2} \sum_{M, M'} \sum_{S, S'}\sum_{j \in M \atop j' \in M'}\sum_{i \in S_1 \atop i' \in S_1'} \mathbb{E}[LL' + LR' + RL' + RR'],
\end{aligned}
\end{equation}
where
\begin{align}
& L = 
\frac{y_i\Indicator(\hat\jmath(M, S_0) =j) \Indicator(x_{ij} \leq \hat s(M,S_0)) \Indicator(x_j \leq  \hat s(M,S_0))}{1+\#\{k \in S_1\backslash \{i\}: x_{kj} \leq \hat s(M,S_0)\}}
,\\
&L'  = 
\frac{y_{i'}\Indicator(\hat \jmath(M',S_0') = j')\Indicator(x_{i'j'} \leq \hat s(M',S_0')) \Indicator(x_{j'} \leq  \hat s(M',S_0'))}{1+\#\{k' \in S_1'\backslash \{i'\}: x_{k'j'} \leq \hat s(M',S_0')\}}
,\\
& R = 
\frac{ y_i\Indicator(\hat\jmath(M, S_0)=j) \Indicator(x_{ij} \geq \hat s(M,S_0)) \Indicator(x_j >  \hat s(M,S_0))}{1+\#\{k \in S_1\backslash \{i\}: x_{kj} > \hat s(M,S_0)\}}
, \; \text{and} \; \\
&R'  =  
\frac{y_{i'} \Indicator(\hat \jmath(M',S_0') = j')\Indicator(x_{i'j'} > \hat s(M',S_0')) \Indicator(x_{j'} \geq  \hat s(M',S_0'))}{1+\#\{k' \in S_1'\backslash \{i'\}: x_{k'j'} > \hat s(M',S_0')\}}.
\end{align}
We evaluate \eqref{eq:mu_mse} by considering five cases on the indices $(i,\,i',\,j,\,j')$.

\subsubsection{Case 1: $i \in S_1 \backslash S_0'$ and $i \not = i'$}
In this case, $y_i$ is independent of $(\{(\bx_k,y_k): k \in S_0 \cup S_0'\},\, \{ \bx_k: k \in S_1 \cup S_1'\},\, y_{i'})$ and $\mathbb{E}[y_i] = 0$, so
we have that 
$
\mathbb{E}[LL'] = \mathbb{E}[LR'] =\mathbb{E}[RL'] = \mathbb{E}[RR'] = 0.
$

\subsubsection{Case 2: $i' \in S_1' \backslash S_0$ and $i \not = i'$}
As with Case 1, we have that
$\mathbb{E}[LL'] = \mathbb{E}[LR'] =\mathbb{E}[RL'] = \mathbb{E}[RR'] = 0$.

\subsubsection{Case 3: $i \in S_1 \cap S_0'$ and $i' \in S_1' \cap S_0$}
By the Cauchy-Schwartz inequality, we have
\begin{equation} \label{eq:CS}
\begin{aligned}
(\mathbb{E}[LL'])^2
&\leq
\mathbb{E}\Bigg[
\frac{y_i^2\Indicator(\hat \jmath(M,S_0)= j) \Indicator(x_{ij} \leq \hat s(M,S_0))\Indicator(x_j \leq \hat s(M,S_0)) }{(1+\#\{k \in S_1\backslash\{i\}: x_{kj} \leq \hat s(M,S_0)\})^2}\Bigg] \\
& \qquad 
\cdot \mathbb{E}\Bigg[\frac{y_{i'}^2\Indicator(\hat \jmath(M',S_0')= j') \Indicator(x_{i'j'} \leq \hat s(M',S_0'))\Indicator(x_{j'} \leq \hat s(M',S_0'))}{(1+\#\{k' \in S_1'\backslash\{i'\}: x_{k'j} \leq \hat s(M',S_0')\})^2} \Bigg] \\
&\leq
\mathbb{E}\Bigg[
\frac{ \Indicator(x_{ij} \leq \hat s(M,S_0)) \Indicator(\hat\jmath(M,S_0)= j)}{1+\#\{k \in S_1\backslash\{i\}: x_{kj} \leq \hat s(M,S_0)\}}\Bigg]
\cdot \mathbb{E}\Bigg[\frac{ \Indicator(x_{i'j'} \leq \hat s(M',S_0')) \Indicator(\hat \jmath(M',S_0')= j') }{1+\#\{k' \in S_1'\backslash\{i'\}: x_{k'j} \leq \hat s(M',S_0')\}} \Bigg],
\end{aligned}
\end{equation}
where we used the fact that $y_i$ is independent of $(\{\bx_{k'}: k' \in S_1\}, \,\hat s(M,S_0), \,\hat \jmath(M,S_0))$ and $y_{i'}$ is independent of $(\{\bx_{k'}: k' \in S_1'\}, \,\hat s(M',S_0'),\, \hat \jmath(M',S_0'))$.
Now, since
$(\{x_{kj}: k \in S_1\}, \,\hat \jmath(M,S_0))$ is independent of $\hat s(M,S_0)$ and $(\{x_{k'j'}: k' \in S_1'\},\, \hat \jmath(M',S_0'))$ is independent of $\hat s(M',S_0')$, by applying Lemma \ref{lmm:bin} to \eqref{eq:CS}, we have
\begin{equation} \label{eq:LL}
\begin{aligned}
\mathbb{E}[L L']
&\leq
\frac{2}{s}\sqrt{\mathbb{P}(\hat \jmath(M,S_0)= j)\mathbb{P}(\hat \jmath(M',S_0')= j')}.
\end{aligned}
\end{equation}
By symmetry,
we have that
\begin{equation}
\begin{aligned}
\mathbb{E}[L L' + LR' + RL' + RR']
\leq \frac{8}{s}\sqrt{\mathbb{P}(\hat \jmath(M,S_0)= j)\mathbb{P}(\hat \jmath(M',S_0')= j')}.
\end{aligned}
\end{equation}
Therefore, by the Cauchy-Schwarz inequality,
\begin{equation}
\begin{aligned}
& \sum_{j \in M \atop j' \in M'}
\sum_{i \in S_1 \cap S_0' \atop i' \in S_1' \cap S_0} \mathbb{E}[L L' + LR' + RL' + RR'] \\ &\leq 
\frac{8|S_1 \cap S_0'| |S_1' \cap S_0|}{s}
\sum_{j \in M \atop j' \in M'}
\sqrt{\mathbb{P}(\hat \jmath(M,S_0)= j)\mathbb{P}(\hat \jmath(M',S_0')= j')} \\
&\leq \frac{8|S_1 \cap S_0'| |S_1' \cap S_0|}{s}
\sqrt{\sum_{j \in M}
\mathbb{P}(\hat \jmath(M,S_0)= j) \sum_{j' \in M'} \mathbb{P}(\hat \jmath(M',S_0')= j')} \\
&= \frac{8|S_1 \cap S_0'| |S_1' \cap S_0|}{s},
\end{aligned}
\end{equation}
so that, by Lemma \ref{lmm:diff_product}, we have
\begin{equation}\label{eq:inoti'}
\begin{aligned}
\frac{1}{\binom{p}{m}^2\binom{n}{s/2}^2\binom{n-s/2}{s/2}^2}\sum_{M,M'}\sum_{S,S'} \sum_{j \in M \atop j' \in M'}
\sum_{i \in S_1 \cap S_0' \atop i' \in S_1' \cap S_0} \mathbb{E}[L L' + LR' + RL' + RR'] 
\leq\frac{s^3 }{2n(n-s/2)}.
\end{aligned}
\end{equation}

\subsubsection{Case 4: $j = j' \in M \cap M'$ and $ i=i'$}
In this case, $i \in S_1 \cap S_1'$ is not in $S_0$ or $S_0'$ so $y_i y_{i'} = y^2_i$ is independent of $(\{\bx_{k}: k \in S_1\}, \, \hat s(M,S_0),\, \hat \jmath(M,S_0),\,\hat s(M',S_0'))$ and $\mathbb{E}[y^2_i] = 1$. Therefore,
\begin{align}
\mathbb{E}[LL'] &\leq
\mathbb{E}\Bigg[\frac{\Indicator(\hat \jmath(M,S_0)= j) \Indicator(x_{ij} \leq \hat s(M,S_0)) \Indicator(x_j \leq \hat s(M,S_0))\Indicator(x_{j} \leq \hat s(M',S_0'))}{1+\#\{k \in S_1\backslash\{i\}: x_{kj} \leq \hat s(M,S_0)\}}\Bigg] \\
&\leq \frac{\mathbb{P}(\hat \jmath(M,S_0) = j,\,x_j \leq \hat s(M,S_0), \; \text{and} \; x_j \leq \hat s(M',S_0'))}{s/2},
\end{align}
where we similarly applied Lemma \ref{lmm:bin}. By symmetry, we have
\begin{equation}
\begin{aligned}
& \sum_{j = j' \in M \cap M'}\, \sum_{
i= i' \in S_1\cap S_1'}\mathbb{E}[LL'+LR'+RL'
+RR'] \\ &
\leq \sum_{j = j' \in M \cap M'}\,\sum_{i=i' \in S_1 \cap S_1'}\frac{\mathbb{P}(\hat \jmath(M,S_0) = j)}{s/2}
\leq \frac{2|S_1 \cap S_1'||M \cap M'|}{sm}.
\end{aligned}
\end{equation}

Applying Lemma \ref{lmm:cap} twice, we see that
\begin{equation}
\begin{aligned}
\frac{1}{\binom{p}{m}^2\binom{n}{s/2}^2\binom{n-s/2}{s/2}^2}\sum_{M,M'}\,\sum_{S,S'} \,\sum_{j = j' \in M \cap M'}\,\sum_{i=i' \in S_1 \cap S_1'}\mathbb{E}[LL'+LR'+RL'+RR']
= \frac{sm}{2np}.
\end{aligned}
\end{equation}

\subsubsection{Case 5: $j \not = j'$ and $ i=i'$}
If $ j \not \in M'$, then $\#\{x_{kj}: k \in S_1 \backslash\{i\}\}$ is independent of $(y_i, \,\hat s(M,S_0), \, \{\hat \jmath(M,S_0) = j\}, \, L', \, y_i)$. Otherwise $\#\{x_{kj}: k \in S_1 \backslash\{S_0' \cup i\}\}$ (which is less than $\#\{x_{kj}: k \in S_1\backslash \{i\} \}$) is independent of $(\hat s(M,S_0), \, \{\hat \jmath(M,S_0) = j\}, \, L')$. Therefore, by applying Lemma \ref{lmm:bin}, we have
\begin{equation}
\begin{aligned}
& \mathbb{E}[L \mid y_i, \, \hat s(M,S_0), \, \{\hat \jmath(M,S_0) = j\},\, L'] \\ & \leq 
y_i \Indicator(\hat \jmath(M,S_0) = j) \Bigg(
\frac{\Indicator(j \not \in M')}{s/2}+\frac{\Indicator(j \in M')}{|S_1 \backslash S_0'|}\Bigg)\Indicator(x_j \leq \hat s(M,S_0)).
\end{aligned}
\end{equation}
Similarly, we also have 
\begin{equation}
\begin{aligned}
& \mathbb{E}[L'\mid y_i, \, \{\hat \jmath(M,S_0) = j\}, \, \hat s(M',S_0'), \, \{\hat \jmath(M',S_0') = j'\}]
 \\ &\leq 
y_{i} \Indicator(\hat \jmath(M',S_0') = j') \Bigg(
\frac{\Indicator(j' \not \in M)}{s/2}+\frac{\Indicator(j' \in M)}{|S_1' \backslash S_0|}\Bigg)
\Indicator(x_{j'} \leq \hat s(M',S_0')).
\end{aligned}
\end{equation}
Therefore, we have
\begin{align}
\mathbb{E}[LL']
&\leq 
\mathbb{P}(\hat\jmath(M, S_0) =j, \, \hat \jmath(M',S_0') = j',\, x_j \leq \hat s(M, S_0), \; \text{and} \;x_{j'} \leq \hat s(M', S_0')) \\
&\qquad\qquad\qquad\qquad\qquad\qquad \cdot \Bigg(\frac{\Indicator(j \not \in M')}{s/2}+\frac{\Indicator(j \in M')}{|S_1 \backslash S_0'|}\Bigg)\Bigg(\frac{\Indicator(j' \not \in M)}{s/2}+\frac{\Indicator(j' \in M)}{|S_1 \backslash S_0'|}\Bigg),
\end{align}
where we used the fact that
$y_i^2$  is independent of the data indices in $S_0 \cup S_0'$, for $i = i' \in S_1 \cap S_1'$,
and
$\mathbb{E}[y_i^2]=1$.
By symmetry, we have
\begin{equation} \label{eq:upper}
\begin{aligned}
&\sum_{j \in M \atop j \in M'}\sum_{i \in S_1 \cap S_1'} \mathbb{E}[LL'+LR'+RL'+RR']\\
&\leq \sum_{j \in M \atop j \in M'}\sum_{i \in S_1 \cap S_1'} \mathbb{P}(\hat\jmath(M, S_0) =j, \, \hat \jmath(M',S_0') = j') \Bigg(\frac{\Indicator(j \not \in M')}{s/2}+\frac{\Indicator(j \in M')}{|S_1 \backslash S_0'|}\Bigg)\Bigg(\frac{\Indicator(j' \not \in M)}{s/2}+\frac{\Indicator(j' \in M)}{|S_1 \backslash S_0'|}\Bigg)\\
&\leq \frac{|S_1 \cap S_1'|}{m^2} \Bigg( \frac{m - |M \cap M'|}{s/2} +  \frac{|M \cap M'|}{|S_1\backslash S_0'|}\Bigg)\Bigg( \frac{m - |M \cap M'|}{s/2} +  \frac{|M \cap M'|}{|S_1'\backslash S_0|}\Bigg)\\
&\leq |S_1 \cap S_1'| \Bigg(\frac{4}{s^2}+\frac{2|M \cap M'|}{sm}\Bigg(\frac{1}{|S_1\backslash S_0'|}+ \frac{1}{|S_1'\backslash S_0|} - \frac{4}{s}\Bigg) \\ & \qquad
+\frac{|M \cap M'|}{m} \Bigg(\frac{1}{|S_1\backslash S_0'|}-\frac{2}{s}\Bigg)\Bigg(\frac{1}{|S_1'\backslash S_0|} - \frac{2}{s} \Bigg)\Bigg).
\end{aligned}
\end{equation}
Since $i \in S_1 \cap S_1'$, we have $|S_1 \cap S_1'|\geq 1$, so by \eqref{eq:upper} and Lemma \ref{lmm:diff},
we have
\begin{equation}\label{eq:jnotj'}
\begin{aligned}
& \frac{\sum_{M,M'}\sum_{S,S'} \sum_{j \not = j'}\sum_{i=i'}\mathbb{E}[LL'+LR'+RL'+RR']}{{p \choose m}^2{n \choose s/2}^2{n-s/2 \choose s/2}^2}
 \\ &\leq \frac{\sum_{M,M'}\sum_{S_1,S_1'} |S_1 \cap S_1'|\Big(\frac{4}{s^2}+ \frac{8n|M \cap M'|}{s^2(n-s+2)m} + \frac{4n^2|M \cap M'|}{s^2(n-s+2)^2m}\Big) }{{p \choose m}^2{n \choose s/2}^2} \\
&\leq \frac{1}{n} + \frac{2m}{(n-s+2)p} +  \frac{nm}{(n-s+2)^2p}\\
&\leq \frac{1}{n}\Bigg(1 + \frac{3m}{p} \Bigg(\frac{n}{n-s+2}\Bigg)^2\Bigg),
\end{aligned}
\end{equation}
where we applied Lemma \ref{lmm:cap} in the second inequality.
Combining Cases 1-5, we have thus shown that
\begin{align} \label{eq:LL}
\mathbb{E}\big[ (\hat \mu(\bx))^2\big]
& \leq 
\frac{1}{n}\Bigg(1 + \frac{sm}{2p} + \frac{3m}{p} \Bigg(\frac{n}{n-s+2}\Bigg)^2 + \frac{s^3}{2(n-s/2)}\Bigg).
 \qedhere
\end{align}
\end{proof}

\bibliographystyle{plainnat}
\bibliography{CKT_2024_CART-Inconsistent--bib}

\end{document}